\def\year{2020}\relax
\documentclass[letterpaper]{article} 
\usepackage{aaai20}  
\usepackage{times}  
\usepackage{helvet} 
\usepackage{courier}  
\usepackage[hyphens]{url}  
\usepackage{graphicx} 
\usepackage{graphicx}  
\usepackage{amsmath,amssymb}
\usepackage{enumerate}
\usepackage[linesnumbered,ruled,vlined]{algorithm2e}
\usepackage{comment}
\usepackage{bm}
\usepackage{dsfont}
\usepackage{amsthm}
\usepackage{anyfontsize}

\newcommand{\Sp}[1]{\left(#1\right)}
\newcommand{\Mp}[1]{\left[#1\right]}
\newcommand{\Bp}[1]{\left\{#1\right\}}
\newcommand{\abs}[1]{\left|#1\right|}
\newcommand{\Norm}[1]{\left\|#1\right\|}
\newcommand{\ve}[1]{\mathbf{#1}}
\newcommand{\E}{\mathbb{E}}
\newcommand{\Prob}{\mathbb{P}}
\newcommand{\R}{\mathbb{R}}

\newcommand{\M}{\mathcal{M}}
\newcommand{\aX}{\abs{\mathcal{X}}}
\newcommand{\aA}{\abs{\mathcal{A}}}
\newcommand{\hatp}{\widehat{\mathcal{P}}}

\newcommand{\nls}{n^l\Sp{h, x, a}}

\DeclareMathOperator*{\argmax}{argmax}

\allowdisplaybreaks

\newtheorem{theorem}{Theorem}
\newtheorem{corollary}{Corollary}
\newtheorem{assumption}{Assumption}
\newtheorem{lemma}{Lemma}
\newtheorem*{lemma*}{Lemma}

\theoremstyle{definition}

\newtheorem{definition}{Definition}

\newcounter{auxFootnote}

\title{Parameterized Indexed Value Function for Efficient Exploration \\ in Reinforcement Learning}
\author{
Tian Tan,\textsuperscript{\rm 1}\thanks{Equal contribution.}\setcounter{auxFootnote}{\value{footnote}}\thanks{Corresponding author.} 
Zhihan Xiong,\textsuperscript{\rm 2}\footnotemark[\value{auxFootnote}] 
Vikranth R. Dwaracherla\textsuperscript{\rm 3}\\ 
\textsuperscript{\rm 1}Department of Civil and Environmental Engineering, Stanford University\\
\textsuperscript{\rm 2}Department of Statistics, Stanford University\\
\textsuperscript{\rm 3}Department of Electrical Engineering, Stanford University\\
\{tiantan, zxiong9, vikranth\}@stanford.edu 
}

\begin{document}

\maketitle

\begin{abstract}
It is well known that quantifying uncertainty in the action-value estimates is crucial for efficient exploration in reinforcement learning. Ensemble sampling offers a relatively computationally tractable way of doing this using randomized value functions. However, it still requires a huge amount of computational resources for complex problems. In this paper, we present an alternative, computationally efficient way to induce exploration using \emph{index sampling}. We use an indexed value function to represent uncertainty in our action-value estimates. We first present an algorithm to learn parameterized indexed value function through a distributional version of temporal difference in a tabular setting and prove its regret bound. Then, in a computational point of view, we propose a dual-network architecture, \emph{Parameterized Indexed Networks (PINs)}, comprising one mean network and one uncertainty network to learn the indexed value function. Finally, we show the efficacy of PINs through computational experiments.


\end{abstract}

\section{Introduction}
Efficient exploration is a long-established problem and an active research area in reinforcement learning (RL). It is well known in the RL community that maintaining an uncertainty estimate in the value functions is crucial for efficient exploration \cite{russo2018tutorial}, \cite{ghavamzadeh2015bayesian}, \cite{osband2018randomized}. Conventionally, dithering methods, such as $\epsilon$-greedy and Boltzmann $\epsilon$-greedy, induce exploration by randomly selecting and experimenting with actions. This is one of the common exploration schemes used in many applications \cite{mnih2015human}, \cite{he2015deep}, \cite{tan2019cooperative}, for its simplicity. However, these schemes do not take uncertainty into account in their value estimates. As a result, they require a huge amount of data to learn a desirable policy through dithering in an environment with sparse and delayed reward signals.

Although there has been a growing interest in quantifying uncertainty, such as dropout \cite{srivastava2014dropout}, \cite{gal2016dropout} and variational inference \cite{fortunato2017noisy}, \cite{touati2018randomized}, these methods are typically not suitable for sequential decision making \cite{osband2018randomized}. Instead, ensemble sampling \cite{osband2016deep}, \cite{lu2017ensemble}, \cite{osband2018randomized}, which represents uncertainty by learning an approximate posterior distribution over value functions via bootstrapping, has achieved remarkable success in RL. Each member in the ensemble, typically a neural network, learns a mapping from the input state or feature space to action values from a perturbed version of the observed data. However, ensemble sampling methods demand learning and maintaining many neural networks in parallel, which can be both computation and memory intensive.

Index sampling \cite{posteriorsampling2019} offers a possibility to distill the ensemble process into a ``single'' network which learns a mapping from the state or feature space \emph{and} a random index space to the action values. The random index $z$, can be thought as an independent random variable drawn from a fixed distribution $p_z$. After learning this mapping, the resulting model can be approximately sampled from the posterior over networks conditioned on the observed data by passing different indices. Therefore, one effectively samples a value function from its posterior by sampling a random index $z$.

In this paper, we consider index sampling via a class of Gaussian indexed value functions as we explicitly \emph{parameterize} the action-value function as $Q_z (x,a) = \nu(x, a) + m(x,a)z$, where $(x, a)$ represents a state-action pair, $z \sim \mathcal{N}(0, 1)$, $v(x,a)$ and $m(x,a)$ are real valued functions representing mean and standard deviation of the action-values or commonly known as $Q$-values. We present an algorithm to learn such a \emph{parameterized indexed value function} via a distributional variation of temporal difference (TD) learning, which we refer to as the \emph{distributional TD} or \emph{Wasserstein TD} when Wasserstein distance is used as the distance metric. We prove that this algorithm enjoys a Bayesian regret bound of $\widetilde{O}\Sp{H^2\sqrt{\aX\aA L}}$ in finite-horizon episodic MDPs, where $H$ is the episode length, $L$ is the number of episodes, and $\aX$ and $\aA$ are the cardinality of the state and action spaces. Then, we propose a dual-network architecture \emph{Parameterized Indexed Networks} (PINs) to  generalize with complex models in deep reinforcement learning. PINs consist of a \emph{mean network} for learning $\nu(x,a)$, and an \emph{uncertainty network} for $m(x,a)$. We demonstrate the efficacy of PINs on two benchmark problems, Deep-sea and Cartpole Swing-up, that require deep exploration from Deepmind \textit{bsuite} \cite{osband2019bsuite}. Our open-source implementation of PINs can be found at https://github.com/tiantan522/PINs.

\section{Related Work}
In the model-based setting, the optimal exploration strategy can be obtained through dynamic programming in Bayesian belief space given a prior distribution over MDPs \cite{guez2012efficient}. However, the exact solution is intractable. An alternative way for efficient exploration is via \emph{posterior sampling}. Motivated by \emph{Thompson Sampling}, the posterior sampling reinforcement learning (PSRL) was first introduced in \cite{strens2000bayesian} mainly as a heuristic method. The theoretical aspects of PSRL were poorly understood until very recently. Osband et al. \cite{osband2017posterior} established an $\widetilde{O}\Sp{H^{1.5}\sqrt{\aX\aA L}}$ Bayesian regret bound for PSRL in finite-horizon episodic MDPs. Note that this result improves upon the best previous Bayesian regret bound of $\widetilde{O}\Sp{H^{1.5}\aX\sqrt{\aA L}}$ for \emph{any} reinforcement learning algorithm. Although PSRL enjoys state-of-the-art theoretical guarantees, it is intractable for large practical systems. Approximations are called for when exact solution is intractable. 

The RLSVI algorithm in \cite{osband2017deep} approximates the posterior sampling for exploration by using randomized value functions sampled from a posterior distribution. However, its performance is highly dependent on choosing a reasonable linear representation of the value function for a given problem. To concurrently perform generalization and efficient exploration with a flexible nonlinear function approximator, \cite{osband2016deep} proposed a bootstrapped Deep Q-Network (BootDQN) which learns and maintains an ensemble of neural networks in parallel, each trained by a perturbed version of the observed data. The issue with bootstrapping is that it possesses no mechanism for estimating uncertainty that does not come from the observed dataset. This restrains BootDQN from performing well in environments with delayed rewards. A simple remedy was proposed recently in \cite{osband2018randomized} where each of the ensemble members is trained and regularized to a prior network/function that is fixed after random initialization. Equivalently, this amounts to \emph{adding} a randomized untrainable prior function to each ensemble member in linear representation. This additive prior mechanism can also be viewed as some form of self-motivation or curiosity to direct learning and exploration if the agent has never observed a reward. 

Index sampling presented in \cite{posteriorsampling2019} in a bandit setting aims to learn a posterior distribution over value functions with an extra random index as input. In this paper, we extend this to RL setting using distributional TD approach. We first present an algorithm for learning a parameterized indexed value function and the corresponding Bayesian regret analysis. The formulation inspired us to design PINs which combines reinforcement learning with deep learning.

\section{Analysis of Parameterized Indexed Value Function} 
We first consider learning a parameterized indexed value function in a tabular setting. The environment is modelled as an episodic Markov Decision Process (MDP).         

\subsection{Markov Decision Process Setup}

Our MDP formulation is adopted from \cite{osband2017deep}, which is stated as the following assumption.
\begin{assumption}
    \label{assume:mdp}
    The MDP $\mathcal{M}=\Sp{\mathcal{S}, \mathcal{A}, H, R, \Prob, \rho}$ is finite-horizon time-inhomogeneous such that the state space can be factorized as $\mathcal{S}=\mathcal{S}_0\cup\mathcal{S}_1\cup\dots\cup\mathcal{S}_{H-1}$ and each $s_h\in \mathcal{S}_h$ can be written as a pair $s_h=\Sp{h, x}, x\in\mathcal{X}$ for each time step $h\in\Bp{0, 1, \dots, H-1}$ with $\abs{\mathcal{X}}<\infty$. Further, we have $\abs{\mathcal{A}}<\infty$ and $\Prob\Sp{s_{h+1}\in\mathcal{S}_{h+1}\mid s_h\in\mathcal{S}_h, a_h}=1$ for any $a_h\in\mathcal{A}$, $h<H-1$ and the MDP will terminate with probability $1$ after taking action $a_{H-1}$. Finally, the reward is always binary, which means that $R\Sp{s, a}\in\Bp{0, 1}$ for any $s\in\mathcal{S}, a \in \mathcal{A}$.
\end{assumption}

Then, for each state-action pair $\Sp{h, x, a}$, the transition leads to an observation of $o=(r, x')$ pair consisting of the next $x' \in \mathcal{X}$ and a reward $r\in\Bp{0, 1}$ after taking action $a$. We denote this pair $o$ as an \emph{outcome} of the transition. Let $\mathcal{P}_{h, x, a}$ be a categorical distribution after taking an action $a$ in state $(h, x)$ over the $2\abs{\mathcal{X}}$ possible outcomes, which are finite because of the binary reward. We make the following assumption in our analysis:  

\begin{assumption}
	\label{assume:dir}
	With the above setup, for each $(h, x, a) \in \{0, 1, \dots, H - 2\} \times \mathcal{X} \times \mathcal{A}$, the outcome distribution is drawn from a Dirichilet prior $\mathcal{P}_{h, x, a}\sim\mathrm{Dirichlet}\Sp{\bm{\alpha}^0_{h, x, a}}$ for $\bm{\alpha}^0_{h, x, a} \in \R_{+}^{2 \abs{\mathcal{X}}}$, and each $\mathcal{P}_{h, x, a}$ is drawn independently. Further, assume that there exists some $\beta\geq 3$ such that $\bm{1}^T\bm{\alpha}^0_{h, x, a}=\beta$ for all $\Sp{h, x, a}$. 
\end{assumption}
 
The Dirichilet distribution is chosen because it is a conjugate prior of categorical distribution. This helps in simplifying our analysis on Bayesian regret bound. Nevertheless, one may extend our analysis to the reward which has a bounded support in $\Mp{0, 1}$ by using techniques from \cite{agrawal2012analysis}.   

We define a \emph{policy} to be a mapping from $\mathcal{S}$ to a probability distribution over $\mathcal{A}$, and denote the set of all policies by $\Pi$. For any MDP $\M$ and policy $\pi \in \Pi$, we can define a value function $V^{\pi}_{\M, h}\Sp{x} = \E_{\M, \pi}[\sum_{t=h}^{H-1} r_{t+1} | x_h = x]$ as the expected sum of rewards starting from $x$ at time step $h$ and following policy $\pi$. Then, the optimal value function can be defined as $V^*_{\M, h}\Sp{x}=\max_{\pi}V^{\pi}_{\M, h}\Sp{x}$, and the optimal state-action value function at time $h$ can be defined accordingly:
\fontsize{10}{10}
\begin{equation}
\label{equ:defQ}
\begin{split}
    &Q^*_{\M, h}\Sp{x, a}=\\
    &\E_{\M}\Mp{r_{h+1}+V^*_{\M, h+1}\Sp{x_{h+1}}\mid x_h=x, a_h=a}
\end{split}
\end{equation}
\normalsize

\subsection{Algorithm for Tabular RL}
In this section, we present our algorithm for learning a parameterized indexed value function in the above tabular setting, which is summarized as Algorithm \ref{algo:wtd}.

Here we explicitly parameterize the state-action value function in episode $l$ as $Q^l_{Z, h}\Sp{x, a}=\nu^l\Sp{h, x, a}+m^l\Sp{h, x, a}Z_{h, x, a}$, where $Z_{h, x, a}\sim\mathcal{N}\Sp{0, 1}$ is assumed to be independent for each state-action pair $\Sp{h, x, a}$ for the ease of analysis. To simplify notation, when $Z$ appears as a subscript of $Q$ function, it implicitly depends on the state-action input of $Q$. With a slight abuse of notation, we sometimes use a lowercase subscript $Q^l_{z,h}\Sp{x, a} = \nu^l(h, x, a) + m^l(h, x, a) z_{h, x, a}$ as a sampled value function, where $z_{h, x, a}$ is a sample from a standard Gaussian. Further, let $\overline{Q}\sim\mathcal{N}\Sp{\bar{\theta}, \sigma_0^2}$ be a prior of the $Q$ function for all $(h, x, a)$. Denote $\mathcal{D}_{h, x, a}^{l-1}=\Bp{(r_{h+1}^j, x_{h+1}^j) \mid x_h^j = x, a_h^j = a, j<l}$, which includes all observed transition outcomes for $(h, x, a)$ up to the start of episode $l$, and $n^l\Sp{h, x, a}=\abs{\mathcal{D}^{l-1}_{h, x, a}}$.

In the spirit of temporal difference learning, $\forall (r, x') \in \mathcal{D}_{h, x, a}^{l-1} \text{ and } n^l\Sp{h, x, a} \neq 0$, we define a learning target for our parameterized indexed value function $Q^l_{Z, h}\Sp{x, a}$ as:
\fontsize{9}{9}
\begin{align*}
y^l(h, x, a, r, x') = r+\frac{\sigma Z}{\sqrt{n^l\Sp{h, x, a}}}+\max_{a'\in\mathcal{A}}Q^l_{\tilde{z}, h+1}\Sp{x', a'}
\end{align*}
\normalsize
where $\tilde{z}_{h+1, x', a'}\sim\mathcal{N}\Sp{0, 1}$ is independently sampled for each $(h+1, x', a')$, $\sigma$ is a positive constant added to perturb the observed reward, and $Z\sim\mathcal{N}\Sp{0, 1}$ is a standard normal random variable. This target is similar to the one used in RLSVI \cite{osband2017deep} except that $Z$ is a random variable and the added noise $\sigma$ is divided by $\sqrt{n^l\Sp{h, x, a}}$. This decay of noise $\sigma$ is needed to show concentration in the distribution of value function as we gather more and more data.

The learning target resembles traditional TD in the sense that it can be viewed as one-step look-ahead. However, a random variable $Z$ is needed in the target for indexing. Specifically, any realization of $Z$ reduces it to a TD error between a current estimate at the sampled index and a perturbed scalar target. To take all values of $Z$ into account, we propose to match their distributions directly, resulting in a distributional version of TD learning. We refer it as \emph{distributional temporal difference}, or more specifically \emph{Wasserstein temporal difference} when $p$-Wasserstein distance is used as a metric. In the tabular setting, this leads to the following loss function:
\fontsize{10}{10}
\begin{equation}
\label{equ:tabular_loss}
    \begin{split}
    &\mathcal{L}_p\Sp{\nu^l\Sp{h, x, a}, m^l\Sp{h, x, a}}=\\
    &\sum_{\Sp{r, x'}\in\mathcal{D}^{l-1}_{h, x, a}} W_p\Sp{Q^l_{Z, h}\Sp{x, a},\ y^l\Sp{h, x, a, r, x'}}^2,
    \end{split}
\end{equation}
\normalsize
where $Q^l_{Z, h}\Sp{x, a}=\nu^l\Sp{h, x, a}+m^l\Sp{h, x, a}Z_{h, x, a}$. Here, $W_p(\cdot, \cdot)$ is the $p$-Wasserstein distance between two distributions. Similarly, we can define a regularization to the prior as  
$\psi_p \Sp{\nu, m}=\beta W_p\Sp{\nu+mZ,\ \overline{Q}}^2$, where $\beta=\frac{\sigma^2}{\sigma_0^2}$. Therefore, when updating the parameters $\Sp{\nu, m}$, one need to solve $\min_{\Sp{\nu, m}}\Bp{\mathcal{L}_p+\psi_p}$.

For two Gaussian random variables $X\sim\mathcal{N}\Sp{\mu_X, \sigma_X^2}$ and $Y\sim\mathcal{N}\Sp{\mu_Y, \sigma_Y^2}$ the 2-Wasserstein distance between them is simply $W_2\Sp{X, Y}^2=\Sp{\mu_X-\mu_Y}^2+\Sp{\sigma_X-\sigma_Y}^2$. Therefore, we can minimize $\mathcal{L}_2 + \psi_2$ exactly by
\begin{equation}
\label{equ:nu_update}
\begin{split}
&\nu^l\Sp{h, x, a}=\\
&\frac{\sum_{\Sp{r, x'}\in\mathcal{D}^{l-1}_{h, x, a}}\Sp{r+\max_{a'\in\mathcal{A}}Q^l_{\tilde{z}, h+1}\Sp{x', a'}}+\beta\bar{\theta}}{n^l\Sp{h, x, a}+\beta}
\end{split}
\end{equation}

\begin{equation}
\label{equ:m_update}
m^l\Sp{h, x, a}=\frac{\sqrt{n^l\Sp{h, x, a}}\sigma+\beta\sigma_0}{n^l\Sp{h, x, a}+\beta}
\end{equation}

By combining these steps, we get Algorithm \ref{algo:wtd} for tabular RL. As a sanity check, we see that the standard deviation term $m^l$ will shrink and decrease to zero as $n^l \rightarrow \infty$, which is achieved by decaying $\sigma$ by $\sqrt{n^l}$ in the learning target.
\begin{algorithm}
    \caption{Tabular Wasserstein TD (WTD)}
    \label{algo:wtd}
    \SetAlgoLined
    For all $l\in\Bp{0, 1, \dots, L}$, $h\in\Bp{0, \dots, H-1}$, $x\in\mathcal{X}$ and $a\in\mathcal{A}$, initialize $\mathcal{D}^l_{h, x, a}=\emptyset$ and $\nu^l\Sp{H, x, a}=m^l\Sp{H, x, a}=0$\\
    \For{$l=1, \dots, L$}{
        Sample $\tilde{z}_{h, x, a}\sim\mathcal{N}\Sp{0, 1}$ for all $x\in\mathcal{X}$, $a\in\mathcal{A}$ and $h\in\Bp{0, \dots, H-1}$\\
        \For{$h=H-1, \dots, 0$}{
            \For{$x\in\mathcal{X}$, $a\in\mathcal{A}$}{
                Update $\nu^l\Sp{h, x, a}$ and $m^l\Sp{h, x, a}$ using equations (\ref{equ:nu_update}) and (\ref{equ:m_update})
            }
        }
        Observe $x_0$\\
        $\mathcal{D}_{h, x, a}^{l}\leftarrow\mathcal{D}_{h, x, a}^{l-1}$ for all $x\in\mathcal{X}$, $a\in\mathcal{A}$ and $h\in\Bp{0, \dots, H-1}$\\
        \For{$h=0, \dots, H-1$}{
            Take action $a_h=\argmax_{a\in\mathcal{A}}Q^l_{\tilde{z}, h}\Sp{x_h, a}$ and observe $\Sp{r_{h+1}, x_{h+1}}$\\
            $\mathcal{D}_{h, x_h, a_h}^{l}\leftarrow\mathcal{D}_{h, x_h, a_h}^{l}\cup\Bp{\Sp{r_{h+1}, x_{h+1}}}$
        }
    }
\end{algorithm}

\subsection{Bayesian Regret Bound}
The following theorem is the main result of our analysis which establishes a Bayesian regret bound for Algorithm \ref{algo:wtd}.                                        
\begin{theorem}
	\label{theo:regret}
	Consider an agent $\mathsf{WTD}$ with infinite buffer, greedy actions and an MDP stated in Assumption \ref{assume:mdp} with planning horizon $H$. Under Assumption \ref{assume:dir} with $\beta\geq 3$, if Algorithm \ref{algo:wtd} is applied with $\sigma^2=3H^2$, $\bar{\theta}=H$ and $\frac{\sigma^2}{\sigma_0^2}=\beta$, for any number of episodes $L\in\mathbb{N}$, we have
	\fontsize{9}{9}
	\begin{align*}
	&\mathrm{BayesRegret}\Sp{\mathsf{WTD}, L}\\
	&\leq 5H^2\sqrt{\beta\abs{\mathcal{X}}\abs{\mathcal{A}}L\log_+\Sp{2\abs{\mathcal{X}}\abs{\mathcal{A}}HL}}\log_+\Sp{1+\frac{L}{\abs{\mathcal{X}}\abs{\mathcal{A}}}}\\
	&=\widetilde{O}\Sp{H^2\sqrt{\abs{\mathcal{X}}\abs{\mathcal{A}}L}}
	\end{align*}
	\normalsize
	where $\widetilde{O}\Sp{\cdot}$ ignores all poly-logarithmic terms and $\log_+\Sp{x}=\max\Bp{1, \log\Sp{x}}$.
\end{theorem}

\begin{proof}
The complete proof is included in the supplemental material C. Here, we provide a sketch of the proof, whose framework takes the analysis of RLSVI in \cite{osband2017deep} as reference. We first define a stochastic Bellman operator as follows.

\subsubsection{Stochastic Bellman Operator}
Based on the updating rules (\ref{equ:nu_update}) and (\ref{equ:m_update}), for a general function $Q\in\R^{\abs{\mathcal{X}}\abs{\mathcal{A}}}$, Algorithm \ref{algo:wtd} defines a functional operator $F_{l, h}$:
\fontsize{9}{9}
\begin{equation*}
    \begin{split}
        &F_{l, h}Q\Sp{x, a}=\frac{\sum_{\Sp{r, x'}\in\mathcal{D}^{l-1}_{h, x, a}}\Sp{r+\max_{a'\in\mathcal{A}}Q\Sp{x', a'}}+\beta\bar{\theta}}{n^l\Sp{h, x, a}+\beta}\\
        &\qquad\qquad\qquad + \frac{\sqrt{n^l\Sp{h, x, a}}\sigma+\beta\sigma_0}{n^l\Sp{h, x, a}+\beta}\cdot Z_{h, x, a}
    \end{split}
\end{equation*}
\normalsize
where $Z_{h, x, a}\sim\mathcal{N}\Sp{0, 1}$ is independent from $Q$. Specifically, with this operator, we have $Q_{Z, h}^l=F_{l, h}Q_{Z', h+1}^l$, $Z, Z'$ are independent.

Recall the definition of $Q^*_{\M, h}$ in equation (\ref{equ:defQ}). We denote $F_{\M, h}$ as the true Bellman operator that satisfies $Q^*_{\M, h}=F_{\M, h}Q^*_{\M, h+1}$.

To prove Theorem \ref{theo:regret}, we resort to the following key lemma proved in \cite{osband2017deep}. 
\begin{lemma}
	\label{lemma:sto_optim}
	Let $\Sp{Q^l_{Z, 0}, \dots, Q^l_{Z, H}}$ be the sequence of state-action value function learned by Algorithm \ref{algo:wtd}, where $Q^l_{Z, H}=\ve{0}$, and $\pi^l$ be the greedy policy based on these $Q$-values. For any episode $l\in\mathbb{N}$, if we have
	\begin{equation}
	\label{equ:optim_1}
	\E\Mp{\max_{a'\in\mathcal{A}}Q^l_{Z, 0}\Sp{x^l_0, a'}}\geq\E\Mp{\max_{a'\in\mathcal{A}}Q^*_{\M, 0}\Sp{x^l_0, a'}}
	\end{equation}
	then, by defining $\Delta_l=V^*_{\M, 0}\Sp{x_0^l}-V^{\pi^l}_{\M, 0}\Sp{x_0^l}$, we can have
    \begin{equation}
    \label{equ:raw_bound_1}
        \begin{split}
            \E\Mp{\Delta_l}\leq\E\Mp{\sum_{h=0}^{H-1}\Sp{\Sp{F_{l, h}-F_{\M, h}}Q^l_{Z, h+1}}\Sp{x_h^l, a_h^l}}
        \end{split}
    \end{equation}
\end{lemma}

In order to use the bound (\ref{equ:raw_bound_1}), it is necessary to verify that the condition (\ref{equ:optim_1}) is satisfied by our Algorithm \ref{algo:wtd}. To achieve this, we resort to the concept of \textit{stochastic optimism}, which is defined as the following:
\subsubsection{Stochastic Optimism}
A random variable $X$ is \textit{stochastically optimistic} with respect to another random variable $Y$, denoted as $X\geq_{SO}Y$, if $\E\Mp{u\Sp{X}}\geq\E\Mp{u\Sp{Y}}$ holds for all convex increasing function $u:\R\mapsto\R$.

If Assumption \ref{assume:dir} holds, by additionally assuming some technical conditions on parameters $\sigma^2$, $\bar{\theta}$ and $\beta$, it is possible to show that
$$Q^l_{Z, 0}\Sp{x, a}\mid\mathcal{H}_{l-1}\geq_{SO}Q^*_{\M, 0}\Sp{x, a}\mid\mathcal{H}_{l-1}$$
for any history $\mathcal{H}_{l-1}=\bigcup_{k=1}^{l-1}\Bp{\Sp{h, x_h^k, a_h^k, r_{h+1}^k}\mid h<H}$ and $\Sp{x,a}\in\mathcal{X}\times\mathcal{A}$.

By using Lemma 2 in \cite{osband2017deep} on preservation of optimism, we can then further show that
\fontsize{9}{9}
$$\E_{\mathcal{H}_{l-1}}\Mp{\max_{a'\in\mathcal{A}}Q^l_{Z, 0}\Sp{x^l_0, a'}}\geq\E_{\mathcal{H}_{l-1}}\Mp{\max_{a'\in\mathcal{A}}Q^*_{\M, 0}\Sp{x^l_0, a'}}$$
\normalsize
Therefore, condition (\ref{equ:optim_1}) can be obtained by simply taking expectation over all possible $\mathcal{H}_{l-1}$.

\subsubsection{Towards Regret Bound}
After verifying condition (\ref{equ:optim_1}), we can then apply bound (\ref{equ:raw_bound_1}) in Lemma \ref{lemma:sto_optim} to our Algorithm \ref{algo:wtd} to get
\begin{align*}
    &\mathrm{BayesRegret}\Sp{\mathsf{WTD}, L}=\E\Mp{\sum_{l=1}^{L}\Delta_l}\\
    &\leq\E\Mp{\sum_{l=1}^{L}\sum_{h=0}^{H-1}\Sp{\Sp{F_{l, h}-F_{\M, h}}Q^l_{Z, h+1}}\Sp{x_h^l, a_h^l}}
\end{align*}

Finally, after some algebraic manipulations, it is possible to bound this term by $\widetilde{O}\Sp{H^2\sqrt{L\aX\aA}}$, which completes the proof.
\end{proof}

\section{Parameterized Indexed Networks}
Previously we have proved the efficiency of Wasserstein TD for learning a parameterized indexed value function in the tabular setting. We next discuss the application of this methodology to deep RL. We note that some of the conditions to derive Theorem \ref{theo:regret} are violated in most cases of deep RL, which puts us in the territory of heuristics. More precisely, technical conditions on parameters $\sigma^2$, $\bar{\theta}$ may not hold, the number of visitations $n^l\Sp{h, x, a}$ cannot be counted in general, and as we will see soon that the independence among sampled $\tilde{z}$'s (in Algorithm \ref{algo:wtd}) required for analysis can be violated owing to algorithmic considerations. Nevertheless, insights from the previous algorithm helped us for our design of a deep learning algorithm that can perform well in practice.

We first recall the essence of index sampling, which aims to effectively sample one value function from posterior by sampling an index $z \sim \mathcal{N}(0, 1)$. To induce temporally consistent exploration or deep exploration \cite{osband2016deep}, the agent needs to follow a greedy policy according to the sampled value function in the next episode of interaction. This insight bears a resemblance to the idea of \emph{Thompson sampling} \cite{thompson1933likelihood} for balancing between exploration and exploitation, and it amounts to sampling one single $z^l$ per episode in index sampling. A general algorithm describing the interaction between an agent with index sampling (IS) and its environment is summarized in Algorithm \ref{algo:live} $\mathsf{live\_IS}$. Note that we use $\mathsf{live\_IS}$ as an \emph{off-policy} algorithm since the agent samples from a replay buffer containing previously observed transitions to obtain its current policy (in \textit{line 3}).    
\begin{algorithm}
    \caption{$\mathsf{live\_IS}$}
    \label{algo:live}
    \SetAlgoLined
    \KwIn{$\mathsf{agent}$, $\mathsf{environment}$}
    \For{$l$ in $\Sp{1, 2, \dots}$}{
        Sample $z^l\sim\mathcal{N}\Sp{0, 1}$\\
		$\mathsf{agent.learn\_from\_buffer()}$\\
		$\mathsf{history\leftarrow environment.reset()}$\\
		\While{$\mathsf{history}$ is not \textbf{terminal}}{
		    $\mathsf{action}\leftarrow\mathsf{agent.act(history)}\Sp{z^l}$\\
		    $\mathsf{history}\leftarrow\mathsf{environment.step(action)}$\\
		}
		$\mathsf{agent.update\_buffer(history)}$
    }
\end{algorithm}

We realize that for any realization of $z^l$ that is fixed within an episode, applying the updating rules for action-value function (\ref{equ:nu_update}) and (\ref{equ:m_update}) in the tabular case then gives
\fontsize{9}{9}
\begin{align*}
    &Q^l_{z^l, h}\Sp{x, a}=\frac{\sum_{\Sp{r, x'}\in\mathcal{D}^{l-1}_{h, x, a}}\Sp{r+Q_{z^l, h+1}^l\Sp{x', \tilde{a}}}+\beta\bar{\theta}}{n^l\Sp{h, x, a}+\beta}\\
    &\qquad+\frac{\sqrt{n^l\Sp{h, x, a}}\sigma+\beta\sigma_0}{n^l\Sp{h, x, a}+\beta}z^l\\
    &=\frac{\sum_{\Sp{r, x'}\in\mathcal{D}^{l-1}_{h, x, a}}\Sp{r+\nu^l\Sp{h+1, x', \tilde{a}}}+\beta\bar{\theta}}{n^l\Sp{h, x, a}+\beta}\\
    &+\frac{\sqrt{n^l\Sp{h, x, a}}\sigma+\beta\sigma_0+\sum_{\Sp{r, x'}\in\mathcal{D}^{l-1}_{h, x, a}}m^l\Sp{h+1, x', \tilde{a}}}{n^l\Sp{h, x, a}+\beta}z^l
\end{align*}
\normalsize
where $\tilde{a}=\argmax_{a'\in\mathcal{A}}Q^l_{z^l, h+1}\Sp{x', a'}$ is the best action with respect to the sampled value function at the next time step. Notice that this update for $Q$ function is equivalent to compute the minimization of the distributional loss (\ref{equ:tabular_loss}) with a \emph{modified} target:
\begin{align*}
     &\Sp{\nu^l\Sp{h, x, a}, m^l\Sp{h, x, a}} = \\
    &\min_{\nu, m} \sum_{\Sp{r, x'}\in\mathcal{D}^{l-1}_{h, x, a}} W_2\Sp{\nu + mZ, \tilde{y}^l}^2 + \psi_2\Sp{\nu, m},
\end{align*}
where
\fontsize{9}{9}
\begin{align*}
& \tilde{y}^l\Sp{h, x, a, r, x'} = r+\frac{\sigma Z}{\sqrt{n^l\Sp{h, x, a}}}+Q^l_{Z, h+1}\Sp{x', \tilde{a}} \\
& = \underbrace{r+ \nu^l(h+1, x', \tilde{a})}_{\text{mean target}} + \underbrace{(\frac{\sigma}{\sqrt{n^l\Sp{h, x, a}}}+ m^l(h+1, x', \tilde{a}))}_{\text{uncertainty target}}Z
\end{align*}
\normalsize
Note that in target $\tilde{y}^l$ the subscript $Z$ is a random variable instead of a realization. We see that the uncertainty measure at the next step $h+1$ is propagated into the uncertainty at the current step $h$. Although derived from a completely different perspective, this is consistent with previous findings in \cite{o2017uncertainty}.

\subsection{Network Design}
According to the above derivation from index sampling, Wasserstein TD for learning a parameterized indexed value function is essentially trying to fit the point estimate $\nu(s, a)$ and the uncertainty estimate $m(s, a)$ to their one-step look-head mean target and uncertainty target respectively as illustrated in $\tilde{y}^l$. When $W_2(\cdot, \cdot)$ is used as metric, this reduces to minimizing two separate squared errors. Hence, we propose to use a dual-network architecture named \emph{Parameterized Indexed Networks (PINs)} for deep RL consisting of a mean network for learning $\nu(s, a)$ and an uncertainty network for learning $m(s,a)$. The two networks are joined by a Gaussian with index $Z\sim\mathcal{N}(0, 1)$.

Let $\theta = (\phi, \omega)$ be parameters of the trainable/online PINs, where $\phi$, $\omega$ are parameters of the mean and the uncertainty network respectively, and let $\tilde{\theta} = (\tilde{\phi}, \tilde{\omega})$ be the parameters of the target nets \cite{mnih2015human}. Then, in $\mathsf{learn\_from\_buffer()}$ of Algorithm \ref{algo:live}, the agent is essentially trying to minimize the following distributional loss:
\fontsize{9}{9}
\begin{align*}
    &\mathcal{L}_{\text{dist}}\Sp{\theta, \tilde{\theta}, \tilde{\mathcal{D}}}= \beta W_2\Sp{Q_{\theta, Z}, \overline{Q}}^2\\
    &+\sum_{\Sp{s, a, r, s'}\in\tilde{\mathcal{D}}}W_2\left(Q_{\theta, Z}\Sp{s, a}, r+\sigma(l)Z+\gamma Q_{\tilde{\theta}, Z}\Sp{s', \bar{a}}\right)^2 
\end{align*}
\normalsize
where $\overline{Q}$ is a prior distribution, $\beta$ is a prior scale which is treated as a hyperparameter, $Z\sim\mathcal{N}\Sp{0, 1}$, $\gamma$ is a discount factor, $\tilde{\mathcal{D}}$ represents sampled minibatch of data consisting of past transitions from a replay buffer, and $\sigma(l)$ is a perturbation added to observed reward $r$. Further, being consistent with our previous analysis, we can consider slowly decaying the added noise variance $\sigma\Sp{l}^2$ after each episode of learning to improve concentration.

For action selection in the learning target, we find empirically that the $\tilde{a}$, obtained by taking argmax on the sampled value function, can sometimes affect learning stability. Therefore, we instead use $\bar{a} \in \argmax_{a'\in\mathcal{A}}\E\Mp{Q_{\tilde{\theta}, Z'}\Sp{s', a'}} = \argmax_{a'\in\mathcal{A}} \nu_{\tilde{\phi}}(s', a')$ for action selection, and the mean network reduces to a Deep Q-Network. We include a performance comparison between $\tilde{a}$ and $\bar{a}$ for action selection in supplemental material B.

\begin{figure}[ht]
    \centering
        \includegraphics[width=\linewidth]{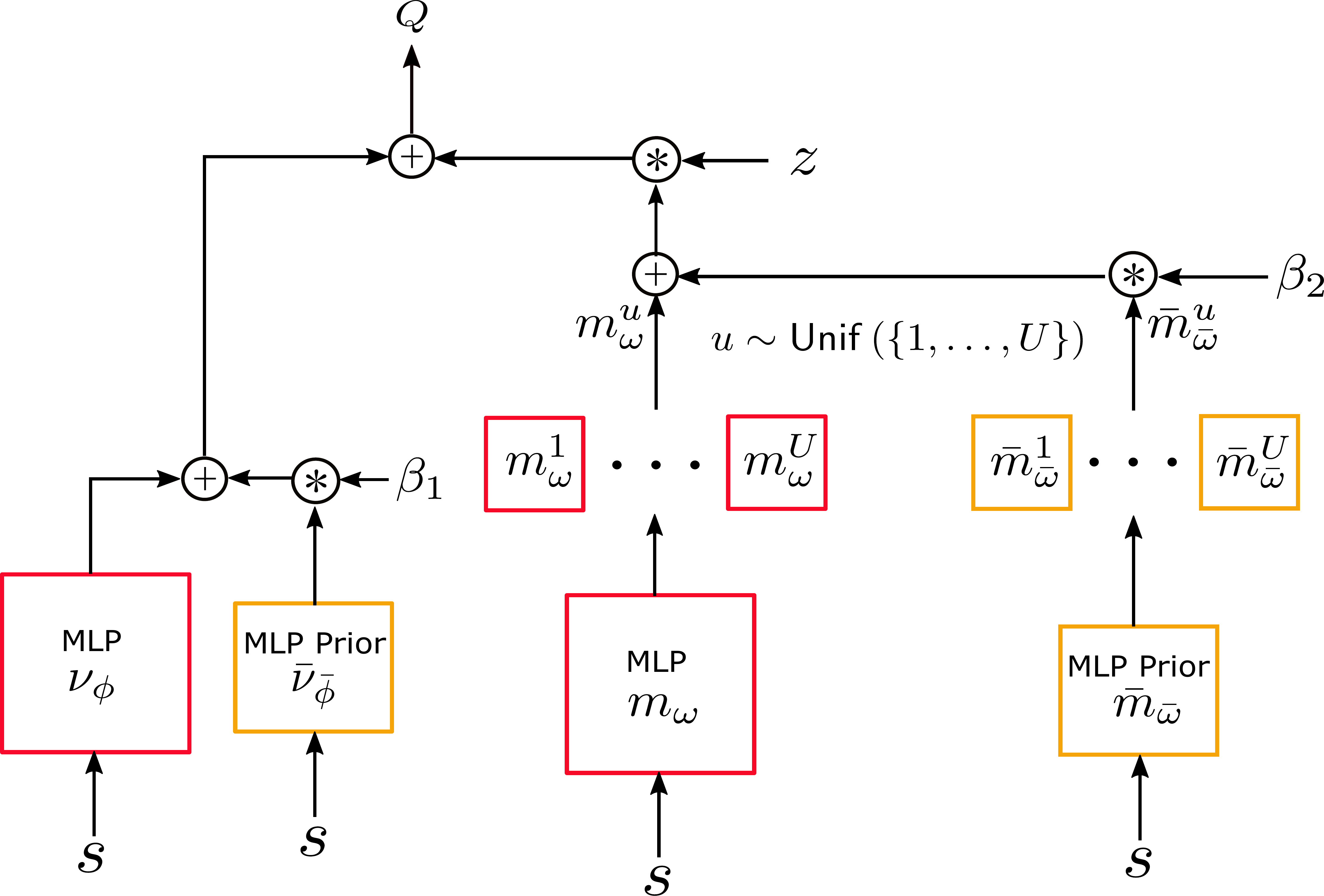}
        \caption{Parameterized Indexed Networks}  
		\label{fig:double_nets}
\end{figure}

\begin{figure*}[ht]
    \centering
    \includegraphics[width=1.0\textwidth]{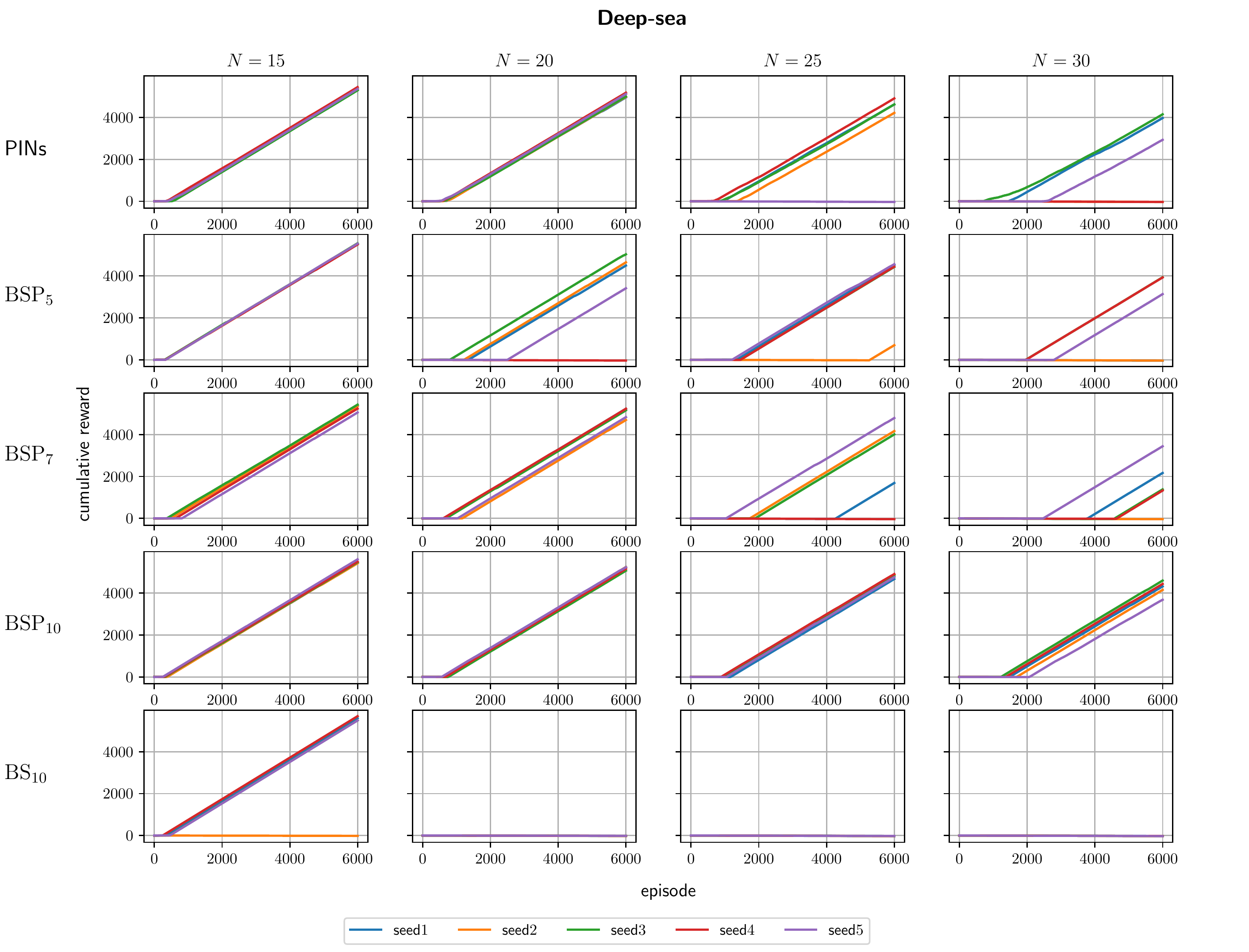}
    \caption{Comparison on cumulative reward with different problem size $N$s among PINs, $\mathrm{BSP}_5$, $\mathrm{BSP}_7$, $\mathrm{BSP}_{10}$ and $\mathrm{BS}_{10}$}
    \label{fig:deep_sea_exp}
\end{figure*}

The overall architecture of PINs is depicted in Figure \ref{fig:double_nets}. We incorporate a prior mechanism by following \cite{osband2018randomized}, where each trainable network is paired with an \emph{additive} prior network. Instead of explicitly regularizing to a prior, i.e. the $\beta W_2\Sp{Q_{\theta, Z}, \overline{Q}}^2$ term in $\mathcal{L}_{\mathrm{dist}}$, we \emph{add} a randomized prior function to the $Q$ function \cite{osband2018randomized}. The prior networks share the exact same architecture as their paired trainable networks and are fixed after random initialization. The structure of the uncertainty network is different from that of the mean network in the following manner:
\begin{itemize}
    \item $\mathsf{Softplus}$ \cite{glorot2011deep} output layer to ensure non-negativity in output 
    values; 
    \item Multiple bootstrapped output heads to encourage diversity in uncertainty estimates: the last hidden layer of the uncertainty net spins out $U$ different output heads. Head $u \in \{1, \dots, U\}$ maps the input state $s$ to uncertainty estimates $m_\omega^u (s, a)$ for all $a \in \mathcal{A}$. Despite that all but the output layer are shared among different heads, we promote diversity by (1) applying a bootstrap mask $\mathsf{Ber}(0.5)$ to each output head so that data in buffer is not completely shared among heads \cite{osband2016deep}, and (2) each head is paired and trained together with a slightly different prior $\bar{m}_{\bar{\omega}}^u(s,a)$. 
\end{itemize}
For each episode $l$, we first sample an index $z^l \sim \mathcal{N}(0, 1)$ and an output head $u \sim \mathsf{Unif}(\{1, \dots, U\})$. The agent then acts greedily with respect to the sampled value function with \emph{additive} priors $Q^u_{z^l} = \nu + m^u z^l + \beta_1 \bar{\nu} + \beta_2 \bar{m}^u z^l$ for consistent exploration in episode $l$, where $\beta_1, \beta_2$ are scaling hyper-parameters for the mean prior and the uncertainty prior, respectively. Here the additive prior distribution can be seen as $\overline{Q}^u_{z^l} = \beta_1 \bar{\nu} + \beta_2 \bar{m}^u z^l$. A detailed training algorithm for PINs is included in supplemental material A.

\section{Experimental Results}

We evaluate the performance of PINs on two benchmark problems, \textit{Deep-sea} and \textit{Cartpole Swing-up}, that highlight the need for deep exploration from Deepmind \textit{bsuite} \cite{osband2019bsuite}, and compare it with the state-of-the-art ensemble sampling methods: the bootstrapped DQN with additive prior networks ($\mathrm{BSP}_K$) \cite{osband2018randomized} and the bootstrapped DQN without prior mechanism ($\mathrm{BS}_{K}$) \cite{osband2016deep}, where $K$ denotes the number of networks in the ensemble.

\subsection{Deep-sea}
Deep-sea is a family of grid-like deterministic environments \cite{osband2017deep}, which are indexed by a problem size $N \in \mathbb{N}$, with $N \times N$ cells as states, and sampled action mask $M_{ij}\sim\mathsf{Ber}(0.5)$, $i, j\in\Bp{1, \dots, N}$. Action set $\mathcal{A} = \{0, 1\}$, and at cell $\Sp{i, j}$, $M_{ij}$ represents action ``left'' and $1-M_{ij}$ represents action ``right''. The agent always starts in the upper-left-most cell at the beginning of each episode. At each cell, action ``left'' (``right'') takes the agent to the cell immediately to the left (right) and below. Thus, each episode lasts exactly $N$ time steps and the agent can never revisit the same state within an episode. No cost or reward is associated with action ``left''. However, taking action ``right'' results in a cost of $0.01/N$ in cells along the main diagonal except the lower-right-most cell where a reward of $1$ is given for taking action ``right''. Therefore, the optimal policy is picking action ``right'' at each step giving an episodic reward of $0.99$. All other policies generate zero or negative rewards. The usual dithering methods will need $\Omega\Sp{2^N}$ episodes to learn the optimal policy, which grows exponentially with the problem size $N$.

Figure \ref{fig:deep_sea_exp} shows the cumulative reward of our PINs and various ensemble models on Deep-sea with four different sizes $N = 15, 20, 25, 30$ for $6K$ episodes of learning. Each approach is evaluated over $5$ different random seeds. We consider that the agent has learned the optimal policy when there is a linear increase in cumulative reward by the end of training. For example, our agent with PINs successfully learned the optimal policy in $3/5$ seeds when $N = 30$, in $4/5$ seeds when $N = 25$ within $6K$ episodes. All networks are MLP with $1$ hidden layer. For PINs, the mean network has $300$ units and the uncertainty network has $512$ units with $U=10$ output heads in the output layer. We set $\sigma = 2$ for the added noise without any decay for experiments on Deep-sea, and $\beta_1 = \beta_2 = 2$. For ensemble models \cite{osband2018randomized}, each single network in the ensemble contains $50$ hidden units, and we set prior scale $\beta=10$ for BSP as recommended in \cite{osband2018randomized}. For BS, we simply let $\beta=0$ to exclude the prior mechanism. For all bootstrapping, we use Bernoulli mask with $p = 0.5$. We see that the performance of PINs is comparable to that of ensemble methods $\mathrm{BSP}_5$ and $\mathrm{BSP}_7$ with additive priors. Also, note that the PINs are relatively more efficient in computation as PINs only require $2$ back-propagations per update while $\mathrm{BSP}_K$ need $K$ backward passes per update. In addition, even equipped with $10$ separate networks, $\mathrm{BS}_{10}$ struggles to learn the optimal policy as $N$ increases, which highlights the significance of a prior for efficient exploration.

Conceptually, the main advantage of our PINs is that it distributes the tasks of learning an value function and measuring uncertainty in estimates into two separate networks. Therefore, it is possible to further enhance exploration by using a more carefully-crafted and more complex design of the uncertainty network and its prior without concerning about the stability of learning in the mean network. As an example, a more delicate design that induces diverse uncertainty estimates for unseen state-action pairs can potentially drive exploration to the next-level. We consider experimenting with different architectures of the uncertainty network as our future work.

\subsection{Cartpole Swing-up}
\begin{figure}[ht]
\centering
\includegraphics[width=1.0\linewidth]{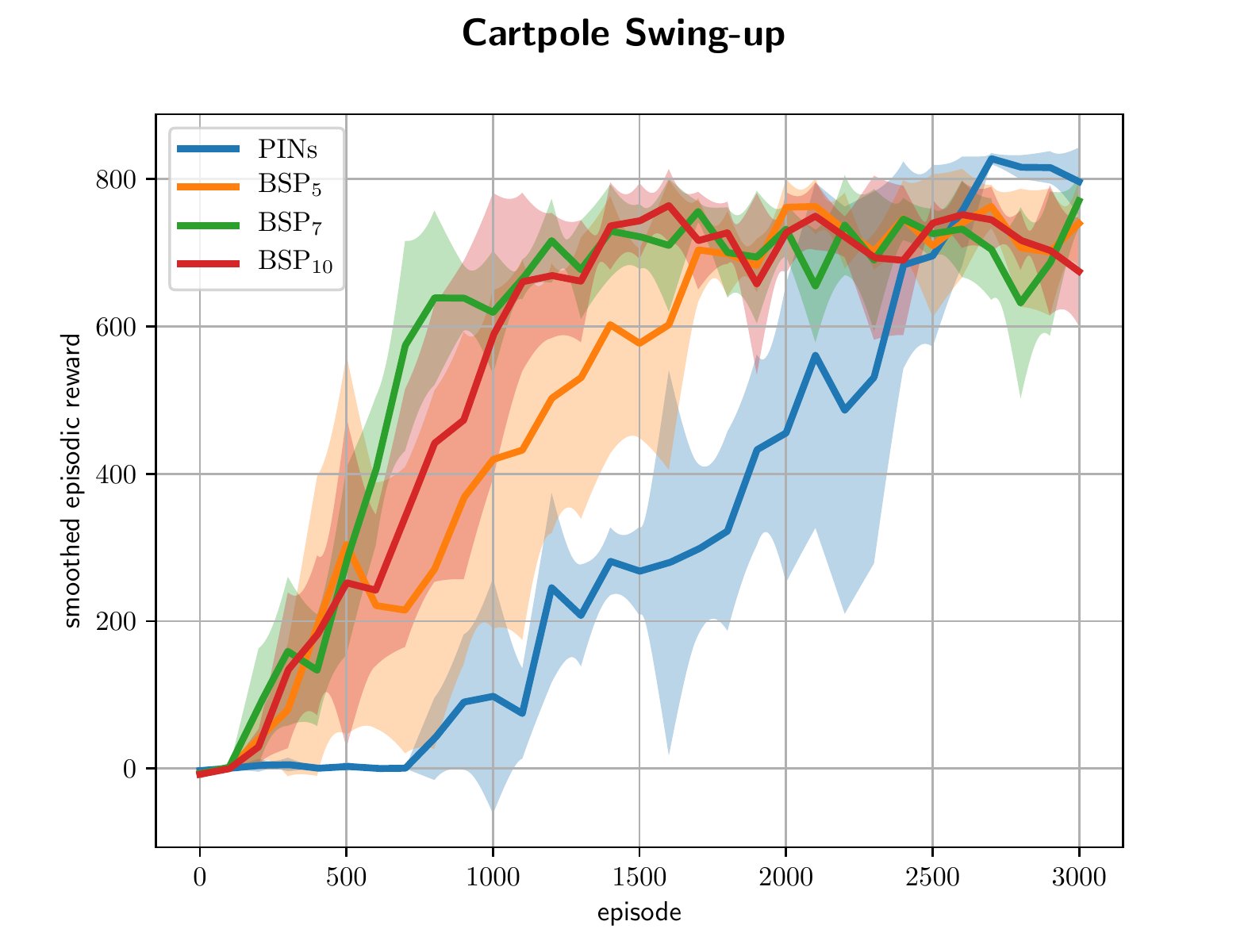}
\caption{Comparison on smoothed episodic reward: each point is the maximum episodic reward within the most recent $100$ episodes. We plot the average performance over $5$ random seeds for each method and the shaded area represents $+/-$ standard deviation.}
\label{fig:cartpole_per}
\end{figure}

We next present our results on a classic benchmark problem: \textit{Cartpole Swing-up} \cite{sutton2018reinforcement}, \cite{osband2019bsuite}, which requires learning a more complex mapping from continuous states\footnote{A $8$ dimensional vector in \textit{bsuite} where we set the threshold of position $x$ to be $5$.} to action values. Unlike the original cartpole, the problem is modified so that the pole is hanging down initially and the agent only receives a reward of $1$ when the pole is nearly upright, balanced, and centered\footnote{Reward $1$ only when $\cos(\theta) > 0.95, \abs{x} < 1, \abs{\Dot{x}} < 1$ and $|\Dot{\theta}| < 1$.} Further, a cost of $0.05$ is added to move the cart and the environment is simulated under timescale of $0.01s$ with a maximum of $1000$ time-steps per episode.

Figure \ref{fig:cartpole_per} shows a comparison on smoothed episodic reward in $3000$ episodes of learning over $5$ different random seeds. We failed to train a $\mathrm{BS}_{10}$ agent that can learn a performant policy on this environment; thus, the results are omitted. To demonstrate computational savings of PINs, all networks used here have $3$ hidden layers with $50$ units in each layer. Besides, the uncertainty network spins out only $U=2$ output heads. For added noise in PINs, we use $\sigma=2$ and linearly decay it to $1$ over the course of training to promote concentration in the approximate posterior distribution. As for prior scale, we use $\beta_1=\beta_2=2$ for our PINs, and $\beta=30$ for BSP as in \cite{osband2018randomized}. We see that PINs achieved similar performance to that of the ensemble models but with only two separate neural networks. Additionally, although PINs seem to progress slowly compared to $\mathrm{BSP}_K$, they exhibit smaller variance in performance by the end of learning. This experiment demonstrates the computational efficiency that can be brought by PINs and by index sampling for representing uncertainty in action-value estimates.

\section{Conclusion}

In this paper, we present a parameterized indexed value function which can be learned by a distributional version of TD. After proving its efficiency in the tabular setting, we introduce a computationally lightweight dual-network architecture, Parameterized Indexed Networks (PINs), for deep RL and show its efficacy through numerical experiments. To the best of our knowledge, we lead the first study of index sampling to achieve efficient exploration in the field of RL.

However, several open questions are still left unanswered. It would be worthwhile to explore other designs for uncertainty and prior networks, and experiment with other distributional metrics to see if one can obtain stronger theoretical guarantees and/or better empirical performance. It would be very7 interesting to combine the ideas of PINs with more advanced neural architectures such as convolutional networks, and evaluate its performance on Atari games with sparse rewards. We leave this and many possible extensions of PINs as our future work, and we hope that this work can serve as a starting point for future studies on the potential of index sampling for efficient RL.

\section{Acknowledgement}
We thank Benjamin Van Roy, Chengshu Li for the insightful discussions, and Rui Du for comments on the earlier drafts.

\medskip

\bibliography{References}
\bibliographystyle{aaai}

\newpage
\appendix 
\onecolumn

\section{Training Algorithm for PINs}

Algorithm \ref{algo:pin} describes the subroutine $\mathsf{agent.learn\_from\_buffer()}$ called in Algorithm 2 $\mathsf{live\_IS}$, where the agent runs several SGD steps to update the mean and the uncertainty network. For each observed transition, we sample a $U$-dimensional binary mask $M$ with $M[j] \sim \mathsf{Ber}(0.5)$ defining data sharing among different output heads in the uncertainty network. In acting during inference, we first randomly sample an index and sample an output head in the uncertainty network to obtain $Q$-values. The agent then acts greedily with respect to the sampled value function during the next episode of interaction.                    

We used a batch size of $64$, and number of batches $N = 10$ for experiments on Deep-sea, $N = 100$ on Cartpole Swing-up. The target networks are updated every $10$ episodes of learning. Adam optimizer with learning rate $10^{-3}$ is used to train both networks. 

\begin{algorithm}[ht]
	\caption{$\mathsf{learn\_from\_buffer}$}
	\label{algo:pin}
	\SetAlgoLined
		\KwIn{$\begin{array}{c}
		\begin{aligned}
		 &\mathsf{buffer}\text{: replay buffer of past transitions} \\
		 & N\text{: number of minibatches}\\
		 & b\text{: batch size}\\
		 &\nu_\phi, m_\omega \text{: trainable networks}\\
		 &\nu_{\tilde{\phi}}, m_{\tilde{\omega}} \text{: target networks, parameters are copied from the trainable nets periodically}\\
		&\bar{\nu}_{\bar{\phi}}, \bar{m}_{\bar{\omega}} \text{: untrainable prior networks} \\
		&\beta_1, \beta_2 \text{: prior scale } \\ &\gamma \text{: discount factor}\\
		&\alpha_\nu, \alpha_m \text{: learning rates} \\
		&\sigma \text{: noise/perturbation scale}
		\end{aligned}
		\end{array}$}
		\For{$n$ in $\Sp{1, 2, \dots, N}$}{
			$\mathcal{D}\leftarrow\mathsf{buffer.sample\_minibatch}\Sp{b}$\\
			Calculate $\bar{a}_i \in \argmax_{a'\in\mathcal{A}} \Sp{\nu_{\tilde{\phi}}+ \beta_1 \bar{\nu}_{\bar{\phi}}}(s'_i, a') $ for each $\Sp{s_i, a_i, r_i, s'_i, M_i}\in\mathcal{D}$\\
			Compute gradient for updating the mean network
			$$\nabla_\phi \mathcal{L}_\nu(\phi, \tilde{\phi}, \bar{\phi}, \mathcal{D}) =  \frac{1}{\abs{\mathcal{D}}}\sum_{\Sp{s_i ,a_i, r_i, s'_i, M_i}\in\mathcal{D}} \nabla_\phi \bigg( \Big(\nu_{\phi}+\beta_1\bar{\nu}_{\bar{\phi}}\Big) \Sp{s_i, a_i}-\Sp{r_i+\gamma\Sp{\nu_{\tilde{\phi}}+\beta_1\bar{\nu}_{\bar{\phi}}}\Sp{s'_i, \bar{a}_i} }\bigg)^2 $$\\
			Randomly select a head: $u_i\sim\mathsf{Unif}\Sp{\mathcal{I}_i}$, where $\mathcal{I}_i=\Bp{j\in\Bp{1, \dots, U}\mid M_i[j]=1}$ for each $\Sp{s_i, a_i, r_i, s'_i, M_i}\in\mathcal{D}$\\
			Compute gradient for updating the uncertainty network
			$$\nabla_\omega \mathcal{L}_m(\omega, \tilde{\omega}, \bar{\omega}, \mathcal{D}) =  \frac{1}{\abs{\mathcal{D}}}\sum_{\Sp{s_i ,a_i, r_i, s'_i, M_i}\in\mathcal{D}} \nabla_\omega \bigg( \Big(m^{u_i}_{\omega}+\beta_2\bar{m}^{u_i}_{\bar{\omega}}\Big) \Sp{s_i, a_i}-\Big(\sigma +\gamma \Big(m^{u_i}_{\tilde{\omega}}+\beta_2\bar{m}^{u_i}_{\bar{\omega}} \Big) \Sp{s'_i, \bar{a}_i} \Big) \bigg)^2 $$\\
			Update
			$$\phi \leftarrow \phi - \alpha_\nu \nabla_\phi \mathcal{L}_\nu (\phi, \tilde{\phi}, \bar{\phi}, \mathcal{D})$$ 
			$$\omega \leftarrow \omega - \alpha_m \nabla_\omega \mathcal{L}_m (\omega, \tilde{\omega}, \bar{\omega}, \mathcal{D})$$\\
			Adjust learning rate by Adam optimizer
		}
\end{algorithm}

\section{Comparison on Different Action Selection Schemes}
We compare two different schemes in selecting actions for training our PINs:  
$$\bar{a} \in  \argmax_{a'\in\mathcal{A}}\E\Mp{Q_{\tilde{\theta}, Z'}\Sp{s', a'} + \bar{Q}_{\bar{\theta}, Z'}\Sp{s', a'} } = \argmax_{a'\in\mathcal{A}} \Bp{\nu_{\tilde{\phi}}(s', a') + \beta_1 \bar{\nu}_{\bar{\phi}} (s', a')}$$ 
as in line 3 of our Algorithm \ref{algo:pin}, where we denote the additive prior distribution as $\bar{Q}_{\bar{\theta}, Z'}\Sp{s, a} = \beta_1 \bar{\nu}_{\bar{\phi}}(s, a) + \beta_2 \bar{m}_{\bar{\omega}}(s,a) Z'$, and $Z'$ is a standard normal random variable. Note that $\bar{a}$ does not depend on which head is picked in the uncertainty network; thus, we omit the superscript for identifying heads on $m(s,a)$ here. 

Alternatively, we can use a sampled value function for action selection:   
$$\tilde{a} \in \argmax_{a' \in \mathcal{A}} \Bp{Q^u_{\tilde{\theta}, z}(s', a') + \bar{Q}^u_{\bar{\theta}, z}(s', a')} = \argmax_{a' \in \mathcal{A}} \Bp{\nu_{\tilde{\phi}}(s', a') + m^u_{\tilde{\omega}}(s',a')z + \beta_1 \bar{\nu}_{\bar{\phi}}(s', a') + \beta_2 \bar{m}^u_{\bar{\omega}}(s', a')z},$$
where $z$ is sampled from $\mathcal{N}(0, 1)$, and $u$ denotes a sampled head in the uncertainty network for a given transition $(s, a, r, s')$ as in line 5 of Algorithm \ref{algo:pin}. We compare $\bar{a}$ and $\tilde{a}$ for training PINs on Deep-sea with $N = 20, 25$. As we see in Figure \ref{fig:compare_action}, $\bar{a}$ gives a more robust performance across different random seeds. 

We also point out that many other options are available for action selection besides $\bar{a}$, $\tilde{a}$. As an example, instead of sampling one single $z$, we may sample multiple $K > 1$ $z$'s independently so that there are $K$ sampled values for each action $a' \in \mathcal{A}$. Then, given the $K$ sampled values, we can consider picking an optimistic estimate for each action before taking the argmax. At the same time, we also need to take the learning stability into account, especially for training the mean network. We believe that this could be an interesting direction to explore in the future work of PINs and of index sampling in general.

\begin{figure}[ht]
    \centering
        \includegraphics[width=\linewidth]{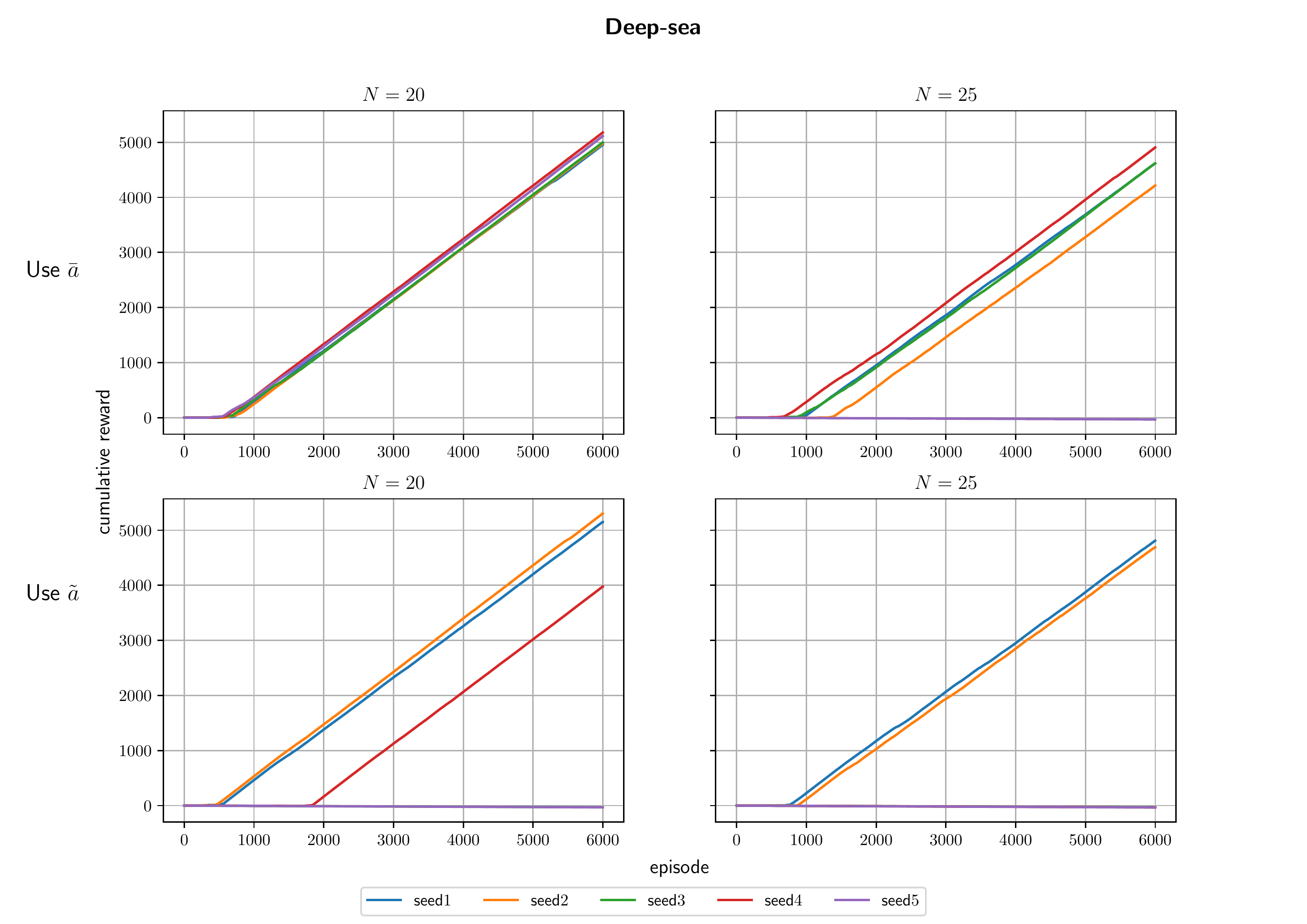}
        \caption{Performance comparison between $\bar{a}$ and $\tilde{a}$ on Deep-sea over $5$ seeds}  
		\label{fig:compare_action}
\end{figure}

\setcounter{assumption}{0}
\setcounter{theorem}{0}
\setcounter{lemma}{0}

\section{Proof of Theorem 1}
For completeness, we first restate the necessary assumptions and theorem here.
\begin{assumption}
    \label{assume:mdp}
    The MDP $\mathcal{M}=\Sp{\mathcal{S}, \mathcal{A}, H, R, \Prob, \rho}$ is finite-horizon time-inhomogeneous such that the state space can be factorized as $\mathcal{S}=\mathcal{S}_0\cup\mathcal{S}_1\cup\dots\cup\mathcal{S}_{H-1}$ and each $s_h\in \mathcal{S}_h$ can be written as a pair $s_h=\Sp{h, x}, x\in\mathcal{X}$ for each $h\in\Bp{0, 1, \dots, H-1}$ with $\abs{\mathcal{X}}<\infty$. Further, we have $\abs{\mathcal{A}}<\infty$ and $\Prob\Sp{s_{h+1}\in\mathcal{S}_{h+1}\mid s_h\in\mathcal{S}_h, a_h}=1$ for any $a_h\in\mathcal{A}$, $h<H-1$ and the MDP will terminate with probability $1$ after taking action $a_{H-1}$. Finally, the reward is always binary, which means that $R\Sp{s}\in\Bp{0, 1}$ for any $s\in\mathcal{S}$.
\end{assumption}
\begin{assumption}
	\label{assume:dir}
	With the above setup, for each $(h, x, a) \in \{0, 1, \dots, H - 2\} \times \mathcal{X} \times \mathcal{A}$, the outcome distribution is drawn from a Dirichilet prior $\mathcal{P}_{h, x, a}\sim\mathrm{Dirichlet}\Sp{\bm{\alpha}^0_{h, x, a}}$ for $\bm{\alpha}^0_{h, x, a} \in \R_{+}^{2 \abs{\mathcal{X}}}$, and each $\mathcal{P}_{h, x, a}$ is drawn independently. Further, assume that there exists some $\beta\geq 3$ such that $\bm{1}^T\bm{\alpha}^0_{h, x, a}=\beta$ for all $\Sp{h, x, a}$. 
\end{assumption}
\begin{theorem}
	\label{theo:regret}
	Consider an agent $\mathsf{WTD}$ with infinite buffer, greedy actions and an MDP stated in Assumption \ref{assume:mdp} with planning horizon $H$. Under Assumption \ref{assume:dir} with $\beta\geq 3$, if Algorithm 1 is applied with $\sigma^2=3H^2$, $\bar{\theta}=H$ and $\frac{\sigma^2}{\sigma_0^2}=\beta$, for any number of episodes $L\in\mathbb{N}$, we have
	\begin{align*}
	\mathrm{BayesRegret}\Sp{\mathsf{WTD}, L}&\leq 5H^2\sqrt{\beta\abs{\mathcal{X}}\abs{\mathcal{A}}L\log_+\Sp{2\abs{\mathcal{X}}\abs{\mathcal{A}}HL}}\log_+\Sp{1+\frac{L}{\abs{\mathcal{X}}\abs{\mathcal{A}}}}\\
	&=\widetilde{O}\Sp{H^2\sqrt{\abs{\mathcal{X}}\abs{\mathcal{A}}L}},
	\end{align*}
	where $\widetilde{O}\Sp{\cdot}$ ignores all poly-logarithmic terms and $\log_+\Sp{x}=\max\Bp{1, \log\Sp{x}}$.
\end{theorem}

\subsection{Stochastic Bellman Operator}
Based on the updating rules for $\nu^l\Sp{h, x, a}$ and $m^l\Sp{h, x, a}$ in the main paper, for a general function $Q\in\R^{\abs{\mathcal{X}}\abs{\mathcal{A}}}$, Algorithm 1 defines a functional operator $F_{l, h}$:
\begin{equation}
\label{equ:operator}
    \begin{split}
        F_{l, h}Q\Sp{x, a}=\frac{\sum_{\Sp{r, x'}\in\mathcal{D}^{l-1}_{h, x, a}}\Sp{r+\max_{a'\in\mathcal{A}}Q\Sp{x', a'}}+\beta\bar{\theta}}{n^l\Sp{h, x, a}+\beta}+ \frac{\sqrt{n^l\Sp{h, x, a}}\sigma+\beta\sigma_0}{n^l\Sp{h, x, a}+\beta}\cdot Z_{h, x, a}
    \end{split}
\end{equation}
where $Z_{h, x, a}\sim\mathcal{N}\Sp{0, 1}$ is independent from $Q$. Specifically, with this operator, we have $Q_{Z, h}^l=F_{l, h}Q_{Z', h+1}^l$, where $Z$ and $Z'$ are independent.

Recall the definition of $Q^*_{\M, h}$ in equation (1) of the main paper. We denote $F_{\M, h}$ as the true Bellman operator that satisfies $Q^*_{\M, h}=F_{\M, h}Q^*_{\M, h+1}$.

To prove Theorem \ref{theo:regret}, we resort to the following key lemma proved in \cite{osband2017deep}. 
\begin{lemma}
	\label{lemma:sto_optim}
	Let $\Sp{Q^l_{Z, 0}, \dots, Q^l_{Z, H}}$ be the sequence of state-action value function learned by Algorithm 1, where $Q^l_{Z, H}=\ve{0}$, and $\pi^l$ be the greedy policy based on them. For any episode $l\in\mathbb{N}$, if we have
	\begin{equation}
	\label{equ:optim}
	\E\Mp{\max_{a'\in\mathcal{A}}Q^l_{Z, 0}\Sp{x^l_0, a'}}\geq\E\Mp{\max_{a'\in\mathcal{A}}Q^*_{\M, 0}\Sp{x^l_0, a'}}
	\end{equation}
	then,
    \begin{equation}
    \label{equ:raw_bound}
        \begin{split}
            \E\Mp{V^*_{\M, 0}\Sp{x_0^l}-V^{\pi^l}_{\M, 0}\Sp{x_0^l}}\leq\E\Mp{\sum_{h=0}^{H-1}\Sp{F_{l, h}Q^l_{Z, h+1}-F_{\M, h}Q^l_{Z, h+1}}\Sp{x_h^l, a_h^l}}
        \end{split}
    \end{equation}
\end{lemma}

In order to use bound (\ref{equ:raw_bound}), it is necessary to verify that our Algorithm 1 satisfies condition (\ref{equ:optim}). To achieve this, we can resort to the concept of \textit{stochastic optimism}, which is defined as the following:
\subsection{Stochastic Optimism}
\begin{definition}
	A random variable $X$ is \textit{stochastically optimistic} with respect to another random variable $Y$, denoted as $X\geq_{SO}Y$, if $\E\Mp{u\Sp{X}}\geq\E\Mp{u\Sp{Y}}$ holds for all convex increasing function $u:\R\mapsto\R$.
\end{definition}

The following two lemmas, which have been proved in \cite{osband2017deep}, will be useful in showing that the Algorithm 1 satisfies condition (\ref{equ:optim}).

\begin{lemma}
	\label{lemma:conv}
	Suppose $\Sp{X_1, \dots, X_n}$ and $\Sp{Y_1, \dots, Y_n}$ are two collections of independent random variables with $X_i\geq_{SO}Y_i$ for each $i\in\Bp{1, \dots, n}$. Then, for any convex increasing function $f:\R^n\mapsto\R$, we have
	$$f\Sp{X_1, \dots, X_n}\geq_{SO} f\Sp{Y_1, \dots, Y_n}$$
\end{lemma}

\begin{lemma}
	\label{lemma:dir_so}
	Let $Y=P^TV$ for fixed $V\in\R^n$ and $P\sim\mathrm{Dirichilet}\Sp{\bm{\alpha}}$ with $\bm{\alpha}\in\R^n_{+}$ and $\bm{\alpha}^T\bm{1}\geq 3$. Let $X\sim\mathcal{N}\Sp{\mu, \sigma^2}$ with $\mu\geq\frac{\bm{\alpha}^TV}{\bm{\alpha}^T\bm{1}}$ and $\sigma^2\geq\frac{3\mathrm{Span}\Sp{V}^2}{\bm{\alpha}^T\bm{1}}$, then $X\geq_{SO} Y$, where $\mathrm{Span}\Sp{V}=\max_iV_i-\min_iV_i$.
\end{lemma}

Let $\mathcal{H}_{l-1}=\bigcup_{k=1}^{l-1}\Bp{\Sp{h, x_h^k, a_h^k, r_{h+1}^k}\mid h<H}$ be the history up to the end of episode $l-1$. Then, by directly combining Lemma \ref{lemma:conv} and Equation \ref{equ:operator}, we can have the following result.
\begin{lemma}
	\label{lemma:mono}
	For two random functions $Q_1, Q_2\in\R^{\abs{\mathcal{X}}\abs{\mathcal{A}}}$, suppose that condition on history $\mathcal{H}_{l-1}$, the entries of $Q_i\Sp{x, a}$ are drawn independently across $x, a$ and independently of $z$ in Equation \ref{equ:operator}. Then, if
	$$Q_1\Sp{x, a}\mid\mathcal{H}_{l-1}\geq_{SO}Q_2\Sp{x, a}\mid\mathcal{H}_{l-1},\qquad\forall\Sp{x, a}\in\mathcal{X}\times\mathcal{A},$$
	we can have
	$$F_{l, h}Q_1\Sp{x, a}\mid\mathcal{H}_{l-1}\geq_{SO}F_{l, h}Q_2\Sp{x, a}\mid\mathcal{H}_{l-1},\qquad\forall\Sp{x, a}\in\mathcal{X}\times\mathcal{A}, t\in\Bp{0, \dots, H-1}, l\in\mathbb{N}$$
\end{lemma}

For convenience, we will first introduce several new notations. 
\begin{definition}
    For any arbitrary state-action value function $Q$, the \textit{induced value function} $V_Q\in\R^{2\aX}$ is defined as $V_Q\Sp{r, x'}:=r+\max_{a'\in\mathcal{A}}Q\Sp{x', a'}$ for all $x'\in\mathcal{X}$ and $r\in\Bp{0, 1}$
\end{definition}
With this definition, we notice that
\begin{equation}
\label{equ:FM}
    F_{\M, h}Q\Sp{x, a}=\E\Mp{r_{h+1}+\max_{a'\in\mathcal{A}}Q\Sp{x_{h+1}, a'}\mid\M, x_{h}=x, a_{h}=a}=V_Q^T\mathcal{P}_{h, x, a},
\end{equation}
where $F_{\M, h}$ is the true Bellman operator so that $Q^*_{\M, h}=F_{\M, h}Q^*_{\M, h+1}$.

Further, Let $\hatp^l_{h, x, a}\in\R^{2\abs{\mathcal{X}}}$ be the empirical distribution over outcomes $\Sp{r, x'}\in\mathcal{D}^{l-1}_{h, x, a}$. With this definition, we can have 
$$\sum_{\Sp{r, x'}\in\mathcal{D}^{l-1}_{h, x, a}}\Sp{r+\max_{a'\in\mathcal{A}}Q\Sp{x', a'}}=n^l\Sp{h, x, a}V_Q^T\hatp^l_{h, x, a}$$

Therefore, the operator $F_{l, h}$ can now be rewritten as
\begin{equation}
\label{equ:new_f_op}
    F_{l, h}Q\Sp{x, a}=\frac{n^l\Sp{h, x, a}V_Q^T\hatp^l_{h, x, a}+\beta\bar{\theta}}{n^l\Sp{h, x, a}+\beta}+\frac{\sqrt{n^l\Sp{h, x, a}}\sigma+\beta\sigma_0}{n^l\Sp{h, x, a}+\beta}\cdot Z_{h, x, a}
\end{equation}

Further, condition on any history $\mathcal{H}_{l-1}$, a standard result about Dirichlet distribution tells us that the posterior of $\mathcal{P}_{h, x, a}$ will be $\mathcal{P}_{h, x, a}|\mathcal{H}_{l-1}\sim\text{Dirichlet}\Sp{\bm{\alpha}^l_{h, x, a}}$, where 
$$\bm{\alpha}^l_{h, x, a}=\bm{\alpha}^0_{h, x, a}+n^l\Sp{h, x, a}\hatp^l_{h, x, a}$$

With the help of these notations, we can conveniently prove the following lemma. 

\begin{lemma}
	\label{lemma:sto_seq}
	Suppose Assumption \ref{assume:dir} holds and Algorithm 1 with parameters $\Sp{\bar{\theta}, \sigma, \sigma_0}$ and $\beta=\frac{\sigma^2}{\sigma_0^2}$ is applied. Then, for any fixed $Q\in\R^{\abs{\mathcal{X}}\abs{\mathcal{A}}}$, episode $l$ with history $\mathcal{H}_{l-1}$ and timestep $h\in\Bp{0, \dots, H-1}$ such that $\sigma^2\geq 3\cdot\mathrm{Span}\Sp{V_Q}^2$ and $\bar{\theta}\geq\Norm{V_Q}_{\infty}$, we can have
	$$F_{l, h}Q\Sp{x, a}\mid\mathcal{H}_{l-1}\geq_{SO}F_{\M, h}Q\Sp{x, a}\mid\mathcal{H}_{l-1},\qquad\forall\Sp{x, a}\in\mathcal{X}\times\mathcal{A}$$
\end{lemma}
\begin{proof}
	Recall that $F_{\M, h}Q\Sp{x, a}=V_Q^T\mathcal{P}_{h, x, a}$. Further, we have $\mathcal{P}_{h, x, a}\mid\mathcal{H}_{l-1}\sim\text{Dirichlet}\Sp{\bm{\alpha}^l_{h, x, a}}$ with $\bm{\alpha}^l_{h, x, a}=\bm{\alpha}^0_{h, x, a}+n^l\Sp{h, x, a}\hatp^l_{h, x, a}$, which implies that $\bm{1}^T\bm{\alpha}^l_{h, x, a}=\beta+n^l\Sp{h, x, a}$ by Assumption \ref{assume:dir}.\\
	Meanwhile, we can also have $F_{l, h}Q\Sp{x, a}\mid\mathcal{H}_{l-1}\sim\mathcal{N}\Sp{\tilde{\mu}, \tilde{\sigma}^2}$, where
	$$\tilde{\mu}=\frac{n^l\Sp{h, x, a}V_Q^T\hatp^l_{h, x, a}+\beta\bar{\theta}}{n^l\Sp{h, x, a}+\beta},\quad, \tilde{\sigma}^2=\Sp{\frac{\sqrt{n^l\Sp{h, x, a}}\sigma+\beta\sigma_0}{n^l\Sp{x, h, a}+\beta}}^2$$
	By Lemma \ref{lemma:dir_so}, we only need to show that $\tilde{\mu}\geq\frac{V_Q^T\bm{\alpha}^l_{h, x, a}}{\bm{1}^T\bm{\alpha}^l_{h, x, a}}$ and $\tilde{\sigma}^2\geq\frac{3\cdot\text{Span}\Sp{V_Q}^2}{\bm{1}^T\bm{\alpha}^l_{h, x, a}}$.\\
	To show the former one, we have
	\begin{align*}
	    \frac{V_Q^T\bm{\alpha}^l_{h, x, a}}{\bm{1}^T\bm{\alpha}^l_{h, x, a}}&=\frac{n^l\Sp{h, x, a}V_Q^T\hatp^l_{h, x, a}+V_Q^T\bm{\alpha}^0_{h, x, a}}{n^l\Sp{h, x, a}+\beta}\\
	    &\leq \frac{n^l\Sp{h, x, a}V_Q^T\hatp^l_{h, x, a}+\Norm{V_Q}_{\infty}\bm{1}^T\bm{\alpha}^0_{h, x, a}}{n^l\Sp{h, x, a}+\beta}\\
	    &\leq\tilde{\mu}\tag{Since we assumed that $\bar{\theta}\geq\Norm{V_Q}_{\infty}$ and $\beta=\bm{1}^T\bm{\alpha}^0_{h, x, a}$}
	\end{align*}
	To show the later one, we notice that
	\begin{align*}
	\Sp{\frac{\sqrt{\nls}\sigma+\beta\sigma_0}{\nls+\beta}}^2&=\frac{\nls\sigma^2+\beta^2\sigma_0^2+2\sqrt{\nls}\beta\sigma\sigma_0}{\Sp{\nls+\beta}^2}\\
	&=\frac{\nls\sigma^2+\beta\sigma^2+2\sqrt{\nls}\beta\sigma\sigma_0}{\Sp{\nls+\beta}^2}\tag{Since $\beta=\frac{\sigma^2}{\sigma_0^2}$}\\
	&=\frac{\sigma^2}{\nls+\beta}+\frac{2\sqrt{\nls}\beta\sigma\sigma_0}{\Sp{\nls+\beta}^2}\\
	&\geq\frac{\sigma^2}{\nls+\beta}\geq\frac{3\cdot\text{Span}\Sp{V_Q}^2}{\nls+\beta}\tag{Since $\sigma^2\geq 3\cdot\text{Span}\Sp{V_Q}^2$}\\
	&= \frac{3\cdot\text{Span}\Sp{V_Q}^2}{\bm{1}^T\bm{\alpha}^l_{h, x, a}}=\tilde{\sigma}^2
	\end{align*}
	Therefore, the proof is complete.
\end{proof}

Notice that in our MDP, by Assumption \ref{assume:mdp}, the planning horizon is $H$ and reward can only take 0 or 1. Thus, we must have $\text{Span}\Sp{V_Q}\leq H$ and $\Norm{V_Q}_{\infty}\leq H$. Then, by combining Lemma \ref{lemma:mono} and Lemma \ref{lemma:sto_seq}, we can have the following result, which has been proved in \cite{osband2017deep}.
\begin{corollary}
	If Assumption \ref{assume:dir} holds and Algorithm 1 is applied with parameters $\sigma^2=3H^2$, $\bar{\theta}=H$ and $\beta=\frac{\sigma^2}{\sigma_0^2}$, we can then have
	$$Q^l_{Z, 0}\Sp{x, a}\mid\mathcal{H}_{l-1}\geq_{SO}Q^*_{\M, 0}\Sp{x, a}\mid\mathcal{H}_{l-1}$$
	for any history $\mathcal{H}_{l-1}$ and pair $\Sp{x, a}\in\mathcal{X}\times\mathcal{A}$.
\end{corollary}

Then, by using Lemma \ref{lemma:conv}, which is applicable since both $Z_{h, x, a}$ and $\mathcal{P}_{h, x, a}$ are distributed independently across different states and actions, we can have
$$\E\Mp{\max_{a'\in\mathcal{A}}Q^l_{Z, 0}\Sp{x^l_0, a'}\mid\mathcal{H}_{l-1}}\geq\E\Mp{\max_{a'\in\mathcal{A}}Q^*_{\M, 0}\Sp{x^l_0, a'}\mid\mathcal{H}_{l-1}}$$
Therefore, condition (\ref{equ:optim}) can be obtained by simply taking expectation over all possible $\mathcal{H}_{l-1}$.

\subsection{Towards Regret Bound}
After verifying condition (\ref{equ:optim}), we can then apply bound (\ref{equ:raw_bound}) in Lemma \ref{lemma:sto_optim} to our Algorithm 1. For simplicity, let
\begin{equation}
\label{equ:simp}
\sigma^l\Sp{h, x, a}=\frac{\sqrt{n^l\Sp{h, x, a}}\sigma+\beta\sigma_0}{n^l\Sp{h, x, a}+\beta},\quad\text{and}\quad w^l\Sp{h, x, a}=\sigma^l\Sp{h, x, a}\cdot Z_{h, x, a}\text{ for }Z_{h, x, a}\sim\mathcal{N}\Sp{0, 1}
\end{equation}
The following three lemmas will be useful in our proof. 
\begin{lemma}
	\label{lemma:sum_bound}
	If $\beta\geq 2$, then with probability 1, we can have the following results
	$$\sum_{l\leq L}\sum_{h< H}\frac{1}{n^l\Sp{h, x_h^l, a_h^l}+\beta}\leq H\aX\aA\log\Sp{1+\frac{L}{\aX\aA}}$$
	$$\sum_{l\leq L}\sum_{h< H}\sqrt{\frac{1}{n^l\Sp{h, x_h^l, a_h^l}+\beta}}\leq 2\sqrt{H^2L\aX\aA}$$
\end{lemma}
\begin{lemma}
	\label{lemma:expect_w}
	For each $h< H$ and $l\leq L$, we can have
	$$\E\Mp{w^l\Sp{h, x_h^l, a_h^l}}\leq\sqrt{2\log\Sp{\aA\aX}\E\Mp{\sigma^l\Sp{h, x_h^l, a_h^l}^2}}$$
\end{lemma}
\begin{lemma}
	\label{lemma:maxV_bound}
	If Algorithm 1 is applied with parameters $\sigma^2=3H^2$, $\bar{\theta}=H$ and $\frac{\sigma^2}{\sigma_0^2}=\beta\geq 3$, we can have
	$$\E\Mp{\max_{l\leq L, h< H}\Norm{V_{Q^l_{Z, h+1}}}_{\infty}}\leq 2H+2H^2\sqrt{\log\Sp{2\aX\aA HL}}$$
\end{lemma}
Among them, Lemma \ref{lemma:sum_bound} and \ref{lemma:expect_w} have been proved in \cite{osband2017deep} and \cite{russo2015much}, respectively. Lemma \ref{lemma:maxV_bound} will be proved later. Now, we start to prove the regret bound.  

\begin{proof}
    By defining $\Delta_l=V^*_{\M, 0}\Sp{x_0^l}-V^{\pi^l}_{\M, 0}\Sp{x_0^l}$, we have
\begin{align*}
    &\mathrm{BayesRegret}\Sp{\mathsf{WTD}, L}=\E\Mp{\sum_{l=1}^{L}\Delta_l}\leq\E\Mp{\sum_{l=1}^{L}\sum_{h=0}^{H-1}\Sp{F_{l, h}Q^l_{Z, h+1}-F_{\M, h}Q^l_{Z, h+1}}\Sp{x_h^l, a_h^l}}
\end{align*}
Then, we can see that
\begin{align*}
    \E\Mp{\Delta_l}&\leq \E\Mp{\sum_{h=0}^{H-1}\Sp{F_{l, h}Q^l_{Z, h+1}-F_{\M, h}Q^l_{Z, h+1}}\Sp{x_h^l, a_h^l}}\\
    &=\E\Mp{\sum_{h=0}^{H-1}\Sp{\Sp{F_{l, h}Q_{Z, h+1}^l}\Sp{x_h^l, a_h^t}-\E\Mp{F_{\M, h}Q^l_{Z, h+1}\Sp{x_h^l, a_h^l}\mid \mathcal{H}_{l-1}, x_0^l, a_-^l, \dots, x_h^l, a_h^l}}}
\end{align*}
To further simplify, recall that $\mathcal{P}_{h, x, a}|\mathcal{H}_{l-1}\sim\text{Dirichlet}\Sp{\bm{\alpha}^l_{h, x, a}}$, where
\begin{equation*}
    \bm{\alpha}^l_{h, x, a}=\bm{\alpha}^0_{h, x, a}+n^l\Sp{h, x, a}\hatp^l_{h, x, a}
\end{equation*}
Notice that the value of $n^l\Sp{h, x, a}$ and $\hatp^l_{h, x, a}$ will not change until we observe the transition outcome of $x_h^l$ and $a_h^l$. Therefore, we have $\mathcal{P}_{h, x, a}|\mathcal{H}_{l-1}, x_0^l, a_0^l, \dots, x_h^l, a_h^l\sim\text{Dirichlet}\Sp{\bm{\alpha}^l_{h, x, a}}$. Thus, for a fixed state-action value function $Q\in\R^{\aX\aA}$, we can have
\begin{align*}
    \E\Mp{F_{\M, h}Q\Sp{x_h^l, a_h^l}\mid\mathcal{H}_{l-1}, x_0^l, a_0^l, \dots, x_h^l, a_h^l}&=V_Q^T\E\Mp{\mathcal{P}_{h, x_h^l, a_h^l}\mid \mathcal{H}_{l-1}, x_0^l, a_0^l, \dots, x_h^l, a_h^l}\tag{By equation (\ref{equ:FM})}\\
    &=\frac{n^l\Sp{h, x_h^l, a_h^l}V_Q^T\hatp^l_{h, x_h^l, a_h^l}+V_Q^T\bm{\alpha}^0_{h, x_h^l, a_h^l}}{n^l\Sp{h, x_h^l, a_h^l}+\beta}
\end{align*}
Then, by using equation (\ref{equ:new_f_op}), we can have
\begin{align*}
    &F_{l, h}Q\Sp{x_h^l, a_h^l}-\E\Mp{F_{\M, h}Q\Sp{x_h^l, a_h^l}\mid\mathcal{H}_{l-1}, x_0^l, a_0^l, \dots, x_h^l, a_h^l}\\
    &=\frac{\beta\bar{\theta}+V_Q^T\bm{\alpha}^0_{h, x_h^l, a_h^l}}{n^l\Sp{h, x_h^l, a_h^l}+\beta}+w^l\Sp{h, x_h^l, a_h^l}\\
    &\leq\frac{\beta\Sp{\bar{\theta}+\Norm{V_Q}_{\infty}}}{n^l\Sp{h, x_h^l, a_h^l}+\beta}+w^l\Sp{h, x_h^l, a_h^l}\tag{Since $\bm{1}^T\bm{\alpha}^l_{h, x_h^l, a_h^l} = \beta$ by Assumption \ref{assume:dir}}
\end{align*}
Notice that condition on $\mathcal{H}_{l-1}$, $F_{\M, h}$ and $Q_{Z, h+1}$ are independent. Therefore, we can have
$$\E\Mp{\Delta_l}\leq \E\Mp{\sum_{h=0}^{H-1}\frac{\beta\Sp{\bar{\theta}+\Norm{V_{Q_{Z, h+1}^l}}_{\infty}}}{n^l\Sp{h, x_h^l, a_h^l}+\beta}+w^l\Sp{h, x_h^l, a_h^l}}$$
As a result, we have
\begin{align*}
    &\E\Mp{\sum_{l=1}^{L}\Delta_l}\leq \underbrace{\E\Mp{\beta\Sp{\bar{\theta} + \max_{l\leq L, h<H}\Norm{V_{Q^l_{Z, h}}}_{\infty}}\sum_{l\leq L, h< H}\frac{1}{n^l\Sp{h, x_h^l, a_h^l}+\beta}}}_{\text{first term}}+\underbrace{\E\Mp{\sum_{l\leq L, h< H}w^l\Sp{h, x_h^l, a_h^l}}}_{\text{second term}}
\end{align*}
These two terms can be bounded separately. To bound the first term, we have
\begin{align*}
	& \E\Mp{\beta\Sp{\bar{\theta}+\max_{l\leq L, h< H}\Norm{V_{Q^l_{Z, h+1}}}_{\infty}}\sum_{h<H, l\leq L}\frac{1}{n^l\Sp{h, x_h^l, a_h^l}+\beta}}\\
	&\quad\leq \beta\Sp{H+\E\Mp{\max_{l\leq L, h< H}\Norm{V_{Q^l_{Z, h+1}}}_{\infty}}}H\aX\aA\log\Sp{1+\frac{L}{\aX\aA}}\tag{By Lemma \ref{lemma:sum_bound} and $\bar{\theta}=H$}\\
	&\quad\leq \beta\Sp{H+2H+2H^2\sqrt{\log\Sp{2\aX\aA HL}}}H\aX\aA\log\Sp{1+\frac{L}{\aX\aA}}\tag{By Lemma \ref{lemma:maxV_bound}}\\
	&\quad\leq 5\beta H^3\aX\aA\sqrt{\log_+\Sp{2\aX\aA HL}}\log\Sp{1+\frac{L}{\aX\aA}}
\end{align*}
To bound the second term, we have
\begin{align*}
	\E\Mp{\sum_{h<H, l\leq L}w^l\Sp{h, x_h^l, a_h^l}}&\leq\sqrt{2\log\Sp{\aX\aA}}\sum_{h<H, l\leq L}\sqrt{\E\Mp{\sigma^l\Sp{h, x_h^l, a_h^l}^2}}\tag{By Lemma \ref{lemma:expect_w}}
\end{align*}
For a sequence of positive numbers $a_1, \dots, a_n$, by Cauchy-Schwartz inequality, we have
\begin{equation}
	\label{equ:cauchy}
	\sum_{i=1}^{n}\sqrt{a_i}\leq\sqrt{n}\sqrt{\sum_{i=1}^{n}a_i}
\end{equation}
Hence, since $\beta=\frac{\sigma^2}{\sigma_0^2}$, by equation (\ref{equ:simp}) and (\ref{equ:cauchy}) (the case when $n=2$), we should have
\begin{equation}
	\label{equ:sigma_bound}
	\sigma^l\Sp{h, x_h^l, a_h^l}=\frac{\sqrt{n^l\Sp{h, x_h^l, a_h^l}}\sigma+\sqrt{\beta}\sigma}{n^l\Sp{h, x_h^l, a_h^l}+\beta}\leq\sqrt{\frac{2\sigma^2}{n^l\Sp{h, x_h^l, a_h^l}+\beta}}
\end{equation}
Therefore, we can have
\begin{align*}
	\sum_{h<H, l\leq L}\sqrt{\E\Mp{\sigma^l\Sp{h, x_h^l, a_h^l}^2}}&\leq\sqrt{HL}\sqrt{\E\Mp{\sum_{h<H, l\leq L}\sigma^l\Sp{h, x_h^l, a_h^l}^2}}\tag{By equation (\ref{equ:cauchy})}\\
	&\leq\sqrt{HL}\sqrt{\E\Mp{\sum_{h<H, l\leq L}\frac{2\sigma^2}{n^l\Sp{h, x_h^l, a_h^l}+\beta}}}\tag{By equation (\ref{equ:sigma_bound})}\\
	&\leq \sqrt{6H^4L\aX\aA\log\Sp{1+\frac{L}{\aX\aA}}}\tag{By Lemma \ref{lemma:sum_bound} and $\sigma^2=3H^2$}
\end{align*}
$$\implies \E\Mp{\sum_{h<H, l\leq L}w^l\Sp{h, x_h^l, a_h^l}}\leq 2H^2\sqrt{3L\aX\aA\log\Sp{\aX\aA}\log\Sp{1+\frac{L}{\aX\aA}}}$$
Therefore, we have the regret bound
\begin{align*}
	\E\Mp{\sum_{l=1}^{L}\Delta_l}&\leq 5\beta H^3\aX\aA\sqrt{\log_+\Sp{2\aX\aA HL}}\log\Sp{1+\frac{L}{\aX\aA}}+2H^2\sqrt{3L\aX\aA\log\Sp{\aX\aA}\log\Sp{1+\frac{L}{\aX\aA}}}\\
	&\leq\Sp{5\beta H^3\aX\aA+2H^2\sqrt{3L\aX\aA}}\sqrt{\log_+\Sp{2\aX\aA HL}}\log_+\Sp{1+\frac{L}{\aX\aA}}
\end{align*}
We then consider two separate cases. \\ 
When $L\geq 25\beta H^2\aX\aA$, which means $5H\sqrt{\beta\aX\aA}\leq \sqrt{L}$, we can have $5\beta H^3\aX\aA\leq H^2\sqrt{\beta L\aX\aA}$. Therefore, we have
\begin{align*}
	\E\Mp{\sum_{l=1}^{L}\ell_l}&\leq\Sp{\sqrt{\beta}+2\sqrt{3}}H^2\sqrt{L\aX\aA\log_+\Sp{2\aX\aA HL}}\log_+\Sp{1+\frac{L}{\aX\aA}}\\
	&\leq 5H^2\sqrt{\beta L\aX\aA\log_+\Sp{2\aX\aA HL}}\log_+\Sp{1+\frac{L}{\aX\aA}}
\end{align*}
When $L\leq 25\beta H^2\aX\aA$, we can just use the naive bound
$$\E\Mp{\sum_{l=1}^{L}\Delta_l}\leq HL\leq H\sqrt{L}\sqrt{25\beta H^2\aX\aA}=5H^2\sqrt{\beta L\aX\aA}$$
In either case, we can conclude that
\begin{align*}
    \mathrm{BayesRegret}\Sp{\mathsf{WTD}, L}=\E\Mp{\sum_{l=1}^{L}\Delta_l}&\leq 5H^2\sqrt{\beta L\aX\aA\log_+\Sp{2\aX\aA HL}}\log_+\Sp{1+\frac{L}{\aX\aA}}\\
    &= \widetilde{O}\Sp{H^2 \sqrt{L\aX\aA}},
\end{align*}
which completes the proof.
\end{proof}

\subsection{Proof of Lemma \ref{lemma:maxV_bound}}
\setcounter{lemma}{7}
\begin{lemma}
    If Algorithm 1 is applied with parameters $\sigma^2=3H^2$, $\bar{\theta}=H$ and $\frac{\sigma^2}{\sigma_0^2}=\beta\geq 3$, we can have
	$$\E\Mp{\max_{l\leq L, h< H}\Norm{V_{Q^l_{Z, h+1}}}_{\infty}}\leq 2H+2H^2\sqrt{\log\Sp{2\aX\aA HL}}$$
\end{lemma}
\begin{proof}
    Define $w_{\max}=\max_{h< H, l\leq L, a\in\mathcal{A}, x\in\mathcal{X}}\abs{w^l\Sp{h, x, a}}$. Then, for any fixed state-action value function $Q\in\R^{\aX\aA}$, by equation (\ref{equ:new_f_op}), we can have
    \begin{align*}
        \Norm{F_{l, h}Q}_{\infty}&\leq\max\Bp{\bar{\theta}, \Norm{V_Q}_{\infty}}+\max_{x\in\mathcal{X}, a\in\mathcal{A}}w^l\Sp{h, x, a}\\
        &\leq\max\Bp{\bar{\theta}, \Norm{Q}_{\infty}}+1+w_{\max}\tag{By Assumptioin \ref{assume:mdp}, we have reward $r\in\Bp{0, 1}$}
    \end{align*}
    Therefore, we can have
    $$\Norm{Q_{Z, H-1}^l}_{\infty}=\Norm{F_{l, H-1}\bm{0}}_{\infty}\leq\bar{\theta}+1+w_{\max}$$
    $$\implies\Norm{Q_{Z, H-2}^l}_{\infty}\leq\max\Bp{\bar{\theta}, \Norm{Q^l_{Z, H-1}}_{\infty}}+1+w_{\max}\leq \bar{\theta}+2\Sp{1+w_{\max}}$$
    By keeping this procedure and applying that $\bar{\theta}=H$, we can have
    $$\E\Mp{\max_{l\leq L, h< H}\Norm{V_{Q^l_{Z, h+1}}}_{\infty}}\leq H + H\Sp{1+\E\Mp{w_{\max}}}=2H+H\E\Mp{w_{\max}}$$
    For the last term above, we can have
	\begin{align*}
	\E\Mp{w_{\max}}&=\E\Mp{\max_{h\leq H, l\leq L, a\in\mathcal{A}, x\in\mathcal{X}}\abs{\sigma^l\Sp{h, x, a}\cdot\frac{w^l\Sp{h, x, a}}{\sigma^l\Sp{h, x, a}}}}\\
	&\leq\E\Mp{\max_{h\leq H, l\leq L, a\in\mathcal{A}, x\in\mathcal{X}}\abs{\frac{w^l\Sp{h, x, a}}{\sigma^l\Sp{h, x, a}}}}\E\Mp{\max_{h\leq H, l\leq L, a\in\mathcal{A}, x\in\mathcal{X}}\sigma^l\Sp{h, x, a}}\\
	&\overset{\text{(a)}}{\leq} \sqrt{\frac{4\sigma^2}{\beta}\log\Sp{2\aX\aA HL}}\\
	&\leq 2H\sqrt{\log\Sp{2\aX\aA HL}}\tag{Since $\beta\geq 3$ and $\sigma^2=3H^2$}
	\end{align*}
	The inequality (a) above is true for two reasons. First, by equation (\ref{equ:sigma_bound}), we can have $\sigma^l\Sp{h, x, a}\leq\sqrt{\frac{2\sigma^2}{\beta}}$ for any $l\in\mathbb{N}$, $h\in\Bp{0, \dots, H-1}$ and $\Sp{x, a}\in\mathcal{X}\times\mathcal{A}$. Second, for independent $X_1, \dots, X_n\sim\mathcal{N}\Sp{0, 1}$, we can have $\E\Mp{\max_i\abs{X_i}}\leq\sqrt{2\log \Sp{2n}}$.\\
	Therefore, we can have the following bound
	$$\E\Mp{\max_{l\leq L, h< H}\Norm{V_{Q^l_{Z, h+1}}}_{\infty}}\leq 2H+H\E\Mp{w_{\max}}\leq 2H+2H^2\sqrt{\log\Sp{2\aX\aA HL}}$$
\end{proof}

\end{document}